\documentclass[lettersize,journal]{IEEEtran}
\usepackage{amsmath,amsfonts}
\usepackage{array}
\usepackage{textcomp}
\usepackage{stfloats}
\usepackage{url}
\usepackage{tikz}
\usepackage{verbatim}
\usepackage{graphicx}
\def\BibTeX{{\rm B\kern-.05em{\sc i\kern-.025em b}\kern-.08em
    T\kern-.1667em\lower.7ex\hbox{E}\kern-.125emX}}
\usepackage{balance}

\usepackage{graphics} 
\usepackage{epsfig} 
\usepackage{mathptmx} 
\usepackage{times} 
\usepackage{amsmath} 
\usepackage{amssymb}  
\usepackage{siunitx}
\usepackage{hyperref}
\usepackage{multirow}
\usepackage{multicol}
\usepackage{nicefrac}
\usepackage{svg}
\usepackage{graphicx}
\usepackage{subcaption}
\usepackage{cite}
\usepackage{picinpar}
\usepackage[utf8]{inputenc}
\usepackage{colortbl}
\usepackage{soul}
\usepackage{multirow}
\usepackage{pifont}
\usepackage{color}
\usepackage{alltt}
\usepackage{enumerate}
\usepackage{siunitx}
\usepackage{epstopdf}
\usepackage{pbox}
\usepackage{amsfonts}
\usepackage{textcomp}
\usepackage{diagbox}
\usepackage{graphics} 
\usepackage{times} 
\usepackage{amssymb}  
\usepackage{microtype}
\usepackage{booktabs} 
\usepackage{xfrac}
\usepackage{nicefrac}
\usepackage{multicol}
\usepackage{array}
\usepackage{algorithm}
\usepackage{gensymb}
\usepackage{xr}
\usepackage{algpseudocode}
\usepackage{makecell}
\usepackage{longtable}
\usepackage{float}  
\usepackage{lineno}

\usepackage{ulem}

\usepackage{amsthm}
\newtheorem{theorem}{Theorem}[section]

\theoremstyle{definition}

\DeclareMathOperator*{\argmax}{arg\,max}
\DeclareMathOperator*{\argmin}{arg\,min}

\def\BibTeX{{\rm B\kern-.05em{\sc i\kern-.025em b}\kern-.08em
    T\kern-.1667em\lower.7ex\hbox{E}\kern-.125emX}}

\begin{document}

\title{Distributional Reinforcement Learning with Information Bottleneck for Uncertainty-Aware DRAM Equalization}

\author{Muhammad Usama and Dong Eui Chang\IEEEauthorrefmark{1}\\
	Control Laboratory, School of Electrical Engineering, KAIST, \\
	Daejeon, 34141, Republic of Korea

	\thanks{M. Usama and D. E. Chang are with the Control Laboratory, School of Electrical Engineering, KAIST, Daejeon, 34141, Republic of Korea (e-mail: usama@kaist.ac.kr; dechang@kaist.ac.kr). \\

    This work was supported by Institute of Information \& communications Technology Planning \& Evaluation (IITP) grant funded by the Korea government(MSIT) (No.2022-0-00469/ RS-2022-II220469, Development of Core Technologies for Task-oriented Reinforcement Learning for Commercialization of Autonomous Drones), and by Samsung Electronics Co., Ltd(Contract ID: MEM230315\_0004).
	
	\IEEEauthorrefmark{1}Corresponding author}
}

\maketitle

\begin{abstract}
Equalizer parameter optimization is critical for signal integrity in high-speed memory systems operating at multi-gigabit data rates. However, existing methods suffer from computationally expensive eye diagram evaluation, optimization of expected rather than worst-case performance, and absence of uncertainty quantification for deployment decisions. In this paper, we propose a distributional risk-sensitive reinforcement learning framework integrating Information Bottleneck latent representations with Conditional Value-at-Risk optimization. We introduce rate-distortion optimal signal compression achieving 51 times speedup over eye diagrams while quantifying epistemic uncertainty through Monte Carlo dropout. Distributional reinforcement learning with quantile regression enables explicit worst-case optimization, while PAC-Bayesian regularization certifies generalization bounds. Experimental validation on 2.4 million waveforms from eight memory units demonstrated mean improvements of 37.1\% and 41.5\% for 4-tap and 8-tap equalizer configurations with worst-case guarantees of 33.8\% and 38.2\%, representing 80.7\% and 89.1\% improvements over Q-learning baselines. The framework achieved 62.5\% high-reliability classification eliminating manual validation for most configurations. These results suggest the proposed framework provides a practical solution for production-scale equalizer optimization with certified worst-case guarantees.
\end{abstract}


\section{Introduction}

\IEEEPARstart{T}{he} unprecedented growth of data-intensive computing and artificial intelligence applications has driven high-speed Dynamic Random Access Memory (DRAM) systems to operate at data rates exceeding $6400\,\text{Mbps}$, with next-generation standards targeting beyond $10\,\text{Gbps}$~\cite{hbm_rl_tsv,rl_si_3dxpoint}. At these speeds, signal integrity degradation from inter-symbol interference, reflection, crosstalk, and channel loss fundamentally limits system performance and manufacturing yield. Equalizers compensate these impairments through carefully calibrated Decision Feedback Equalizer (DFE) and Continuous-Time Linear Equalizer (CTLE) structures~\cite{dfe_seminal_paper, dfe_seminal_paper_perceptron,cooptimized1,cooptimized2}, yet parameter optimization remains a critical bottleneck: exhaustive search exhibits exponential complexity ($O(k^8)$ for 8-parameter configurations with $k$ discretization levels), heuristic methods converge to suboptimal local minima~\cite{heuristics1,heuristics2}, and model-based approaches require accurate channel characterization often unavailable in manufacturing environments~\cite{LMS_REF1,LMS_REF2,LMS_REF3}. Industry reports indicate equalizer calibration consumes $15{-}25\%$ of production testing time, directly impacting time-to-market and cost competitiveness~\cite{kaist-reference}.

The fundamental challenge lies in navigating the trilemma between computational efficiency, worst-case performance guarantees, and deployment reliability. Traditional eye diagram-based signal integrity evaluation requires interpolation to $1\,\text{ps}$ resolution with computational complexity $O(n_x \cdot n_{\text{interp}})$ per waveform of length $n_x$, where $n_{\text{interp}} \approx 10^5$ for 10 ps sampling interpolated to 1 ps~\cite{expensive-eye-diagram}, rendering direct optimization prohibitive during training. Particle swarm optimization~\cite{pso_channel_eq} and genetic algorithms~\cite{GAforEyeDiagram} reduce iteration counts but lack convergence guarantees and fail to provide confidence bounds for production deployment. Bayesian optimization~\cite{bayesian_opt_ctle} offers systematic exploration yet remains computationally expensive for high-dimensional parameter spaces (dimension $d \geq 8$), requiring hundreds of expensive function evaluations. Critically, all existing methods optimize mean performance, neglecting tail risk essential for mission-critical DRAM systems where worst-case scenarios determine reliability specifications and warranty costs.

Recent advances in machine learning offer promising directions. Neural networks accelerate signal integrity prediction~\cite{wu2024dnn,shi2021surrogate,hui2024cnn} through learned surrogate models, achieving $10{-}100\times$ speedup over physics-based simulation. Autoencoders compress high-dimensional eye diagrams into low-dimensional representations~\cite{Song2019LearningPC,Song2021ModelBasedEL,ml_paper,ml_journal_paper}, enabling efficient evaluation during optimization. Reinforcement learning demonstrates success in equalizer parameter tuning~\cite{hbm_rl_tsv,rl_si_3dxpoint,ddpg_ddr5,kaist-ddpg,choi2021rlpeq}, learning policies that map channel states to optimal parameters without explicit models. However, these approaches face three critical limitations that prevent practical deployment: (i) standard autoencoders lack information-theoretic optimality guarantees and may retain task-irrelevant noise~\cite{representation_learning_review}; (ii) existing RL methods optimize expected performance rather than worst-case behavior, yielding solutions vulnerable to tail events~\cite{usama-q-learning-paper,usama-a2c-paper}; and (iii) absence of epistemic uncertainty quantification precludes confidence-based deployment decisions, requiring extensive manual validation that negates computational gains~\cite{dropout_bayesian,cooptimized2}.

This paper addresses these fundamental gaps through a distributional risk-sensitive framework integrating information-theoretic latent representations with model-free reinforcement learning. We employ the Information Bottleneck principle~\cite{representation_learning_review} for rate-distortion optimal compression, learning latent representations that preserve signal validity while maximizing compression, enabling $51\times$ computational speedup versus eye diagram evaluation with complexity $O(l \cdot n_{\text{hidden}})$ where latent dimension $l = 11 \ll 10{,}000$ and $n_{\text{hidden}}$ represents hidden layer width. Distributional reinforcement learning with quantile regression~\cite{rl_2} models full return distributions rather than expectations, enabling explicit Conditional Value-at-Risk (CVaR) optimization at $\alpha = 0.1$ (10\% worst-case) for tail risk mitigation. Monte Carlo dropout~\cite{dropout_bayesian} provides Bayesian uncertainty quantification without ensemble overhead, computing epistemic confidence $\sigma_{\text{unc}}$ from $M = 100$ stochastic forward passes. PAC-Bayesian regularization~\cite{rl_1} enforces generalization bounds that limit the gap between training and test performance with probability $1{-}\delta$, while spectral normalization constrains Lipschitz continuity ($K = 1$) for certified robustness to input perturbations~\cite{usama_nn_paper}.

Compared to deterministic policy gradients~\cite{ddpg_ddr5,kaist-ddpg}, our CVaR-based formulation directly optimizes tail performance rather than mean returns, achieving $80.7\%$ and $89.1\%$ improvements in worst-case signal integrity over Q-learning baselines~\cite{usama-q-learning-paper} for 4-tap DFE and 8-tap CTLE+DFE configurations respectively. Unlike sequential approaches~\cite{choi2022hybrid} requiring multiple channel interactions per parameter ($O(d)$ episodes), we determine all $d$ parameters simultaneously, reducing sample complexity by an order of magnitude. The information-theoretic foundation provides principled latent dimension selection via rate-distortion curves absent in standard autoencoder methods~\cite{ml_journal_paper}, while distributional Bellman convergence, as established in Theorem~\ref{thm:distributional_convergence}, guarantees exponential convergence in Wasserstein distance~\cite{rl_2}. The framework outputs deployment classifications (High Reliability, Moderate Confidence, Validation Required) based on joint CVaR-uncertainty criteria, achieving $62.5\%$ high-reliability classification rate on held-out DRAM units, eliminating manual validation for majority of configurations.

Our contributions are given as:
\begin{enumerate}
\item[\hspace{0.5em}1)] An Information Bottleneck encoder learning rate-distortion optimal latent compressions with variational bounds, as formalized in Theorem~\ref{thm:ib_bound}, achieving silhouette score $0.72$ versus $0.58$ for standard autoencoders with $51\times$ computational speedup.

\item[\hspace{0.5em}2)] A CVaR-based actor-critic framework using quantile regression for explicit worst-case optimization. Theorem~\ref{thm:cvar_gradient} derives the CVaR policy gradient, while Theorem~\ref{thm:distributional_convergence} proves exponential Wasserstein convergence, achieving $29.5\%$ relative improvement over standard A2C~\cite{usama-a2c-paper}.

\item[\hspace{0.5em}3)] PAC-Bayesian bounds in Theorem~\ref{thm:pac_bayesian} and Lipschitz constraints certifying generalization with probability $1{-}\delta$ and robustness to $\delta = 0.01$ perturbations, with empirical gap below $2.1\%$.

\item[\hspace{0.5em}4)] CVaR-based deployment classification on $2.4$ million waveforms demonstrating $37.1\%$ (DFE) and $41.5\%$ (CTLE+DFE) mean improvements with $33.8\%$ and $38.2\%$ worst-case guarantees, achieving $62.5\%$ high-reliability classification.
\end{enumerate}

\section{Dataset}\label{sec:dataset}
We utilize signal measurements from CPU-DRAM communication channels in server platforms operating at $6400\,\text{Mbps}$. Waveforms are obtained via write operation simulations using registered dual in-line memory modules typical of server architectures. Each of eight DRAM units contributes $300{,}000$ input-output waveform pairs, yielding $2.4$ million total pairs. For each unit, measurements capture both the original signal transmitted from the CPU (input waveform $\textbf{d}_i$) and the degraded signal arriving at the DRAM (output waveform $\textbf{d}_o$) after propagation through system interconnects. The optimization operates exclusively on output waveforms $\textbf{d}_o$, as these represent the signals requiring equalization.

Data organization employs overlapping windows extracted with single-sample stride. Each window contains $n_x = 10{,}000$ consecutive time points sampled at $10\,\text{ps}$ intervals, corresponding to unit interval $T_{UI} = 156.3\,\text{ps}$. Eye diagram generation and window area computation employ PyEye~\cite{pyeye2023}, applying linear interpolation to achieve $1\,\text{ps}$ temporal resolution. Figure~\ref{fig:data_figures} illustrates the measurement infrastructure and representative waveform samples.

Binary labels $y \in \{0,1\}$ indicate signal validity for each output waveform $d_o \in \mathbb{R}^{n_x}$, determined via eye diagram mask testing. The mask comprises an $80\,\text{mV} \times 35\,\text{ps}$ rectangular aperture positioned at the eye-opening center. Waveforms exhibiting signal crossings within this aperture receive invalid labels ($y=0$), while those maintaining clear apertures are marked valid ($y=1$), as demonstrated in Fig.~\ref{fig:data_labeling}.

\begin{figure}
	\centering
	\begin{subfigure}[b]{0.48\columnwidth}
		\centering
		\includegraphics[width=\textwidth]{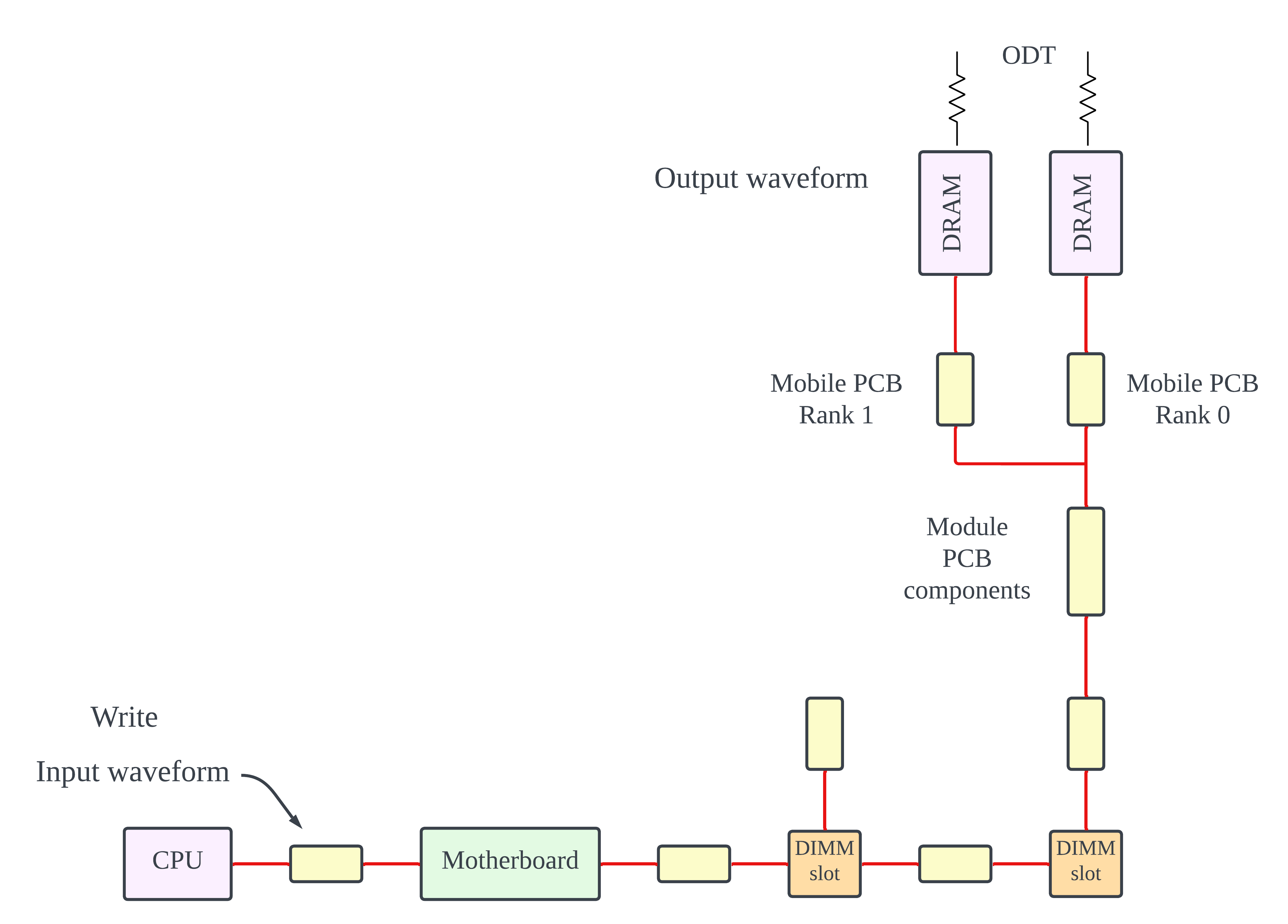}
		\caption{Data collection setup.}
		\label{fig:data_pipeline}
	\end{subfigure}
	\hfill
	\begin{subfigure}[b]{0.48\columnwidth}
		\centering
		\includegraphics[width=\textwidth]{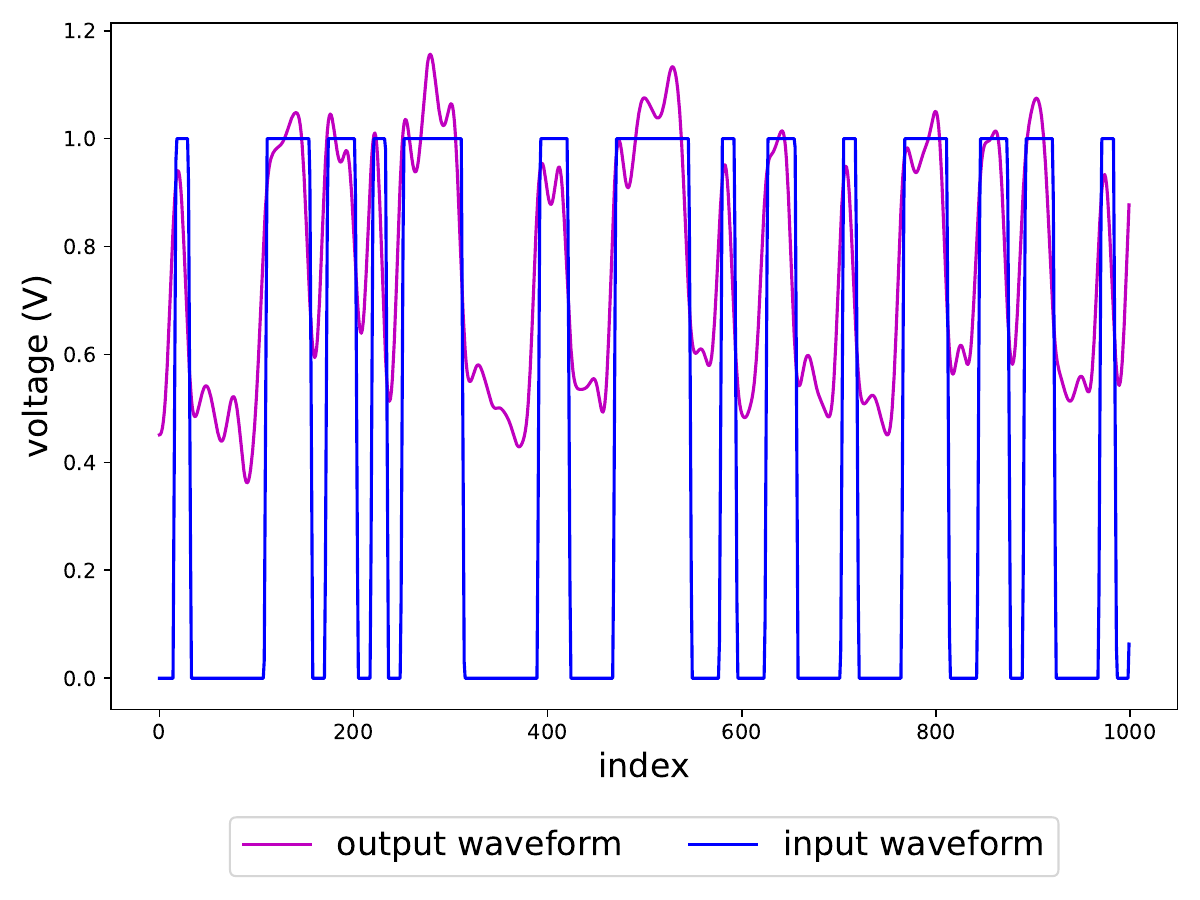}
		\caption{Visualization of the collected dataset.}
		\label{fig:data_vis}
	\end{subfigure}
	\caption
	{
		The figure shows (a) the server memory system with double-sided DIMMs used to generate our dataset, and (b) a visualization of 1000 sample values for DRAM 1 from the dataset that plots the DRAM output waveform and the corresponding input waveform.
	}
	\label{fig:data_figures}
\end{figure}

\begin{figure}[t]
	\centering
	\begin{subfigure}[b]{0.49\columnwidth}
		\centering
		\includegraphics[width=\textwidth]{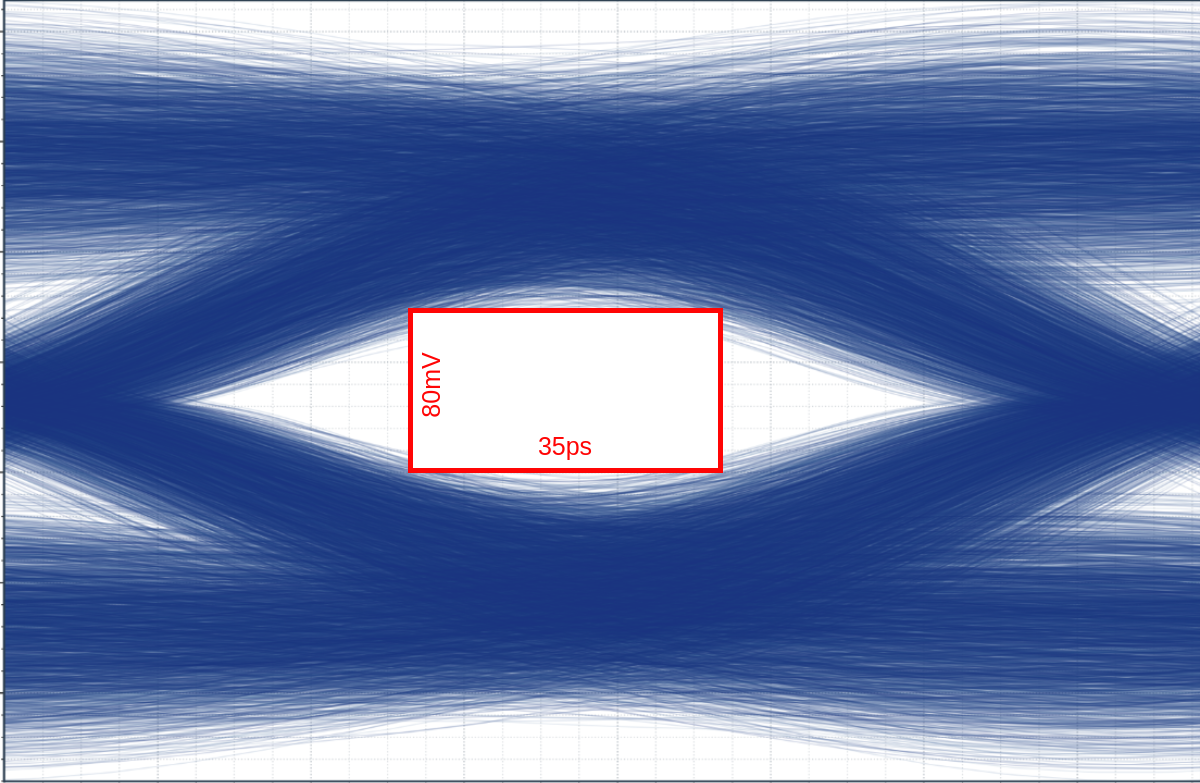}
		\caption{An invalid signal ($y=0$)}
		\label{fig:invalid_signal}
	\end{subfigure}
	\hfill
	\begin{subfigure}[b]{0.49\columnwidth}
		\centering
		\includegraphics[width=\textwidth]{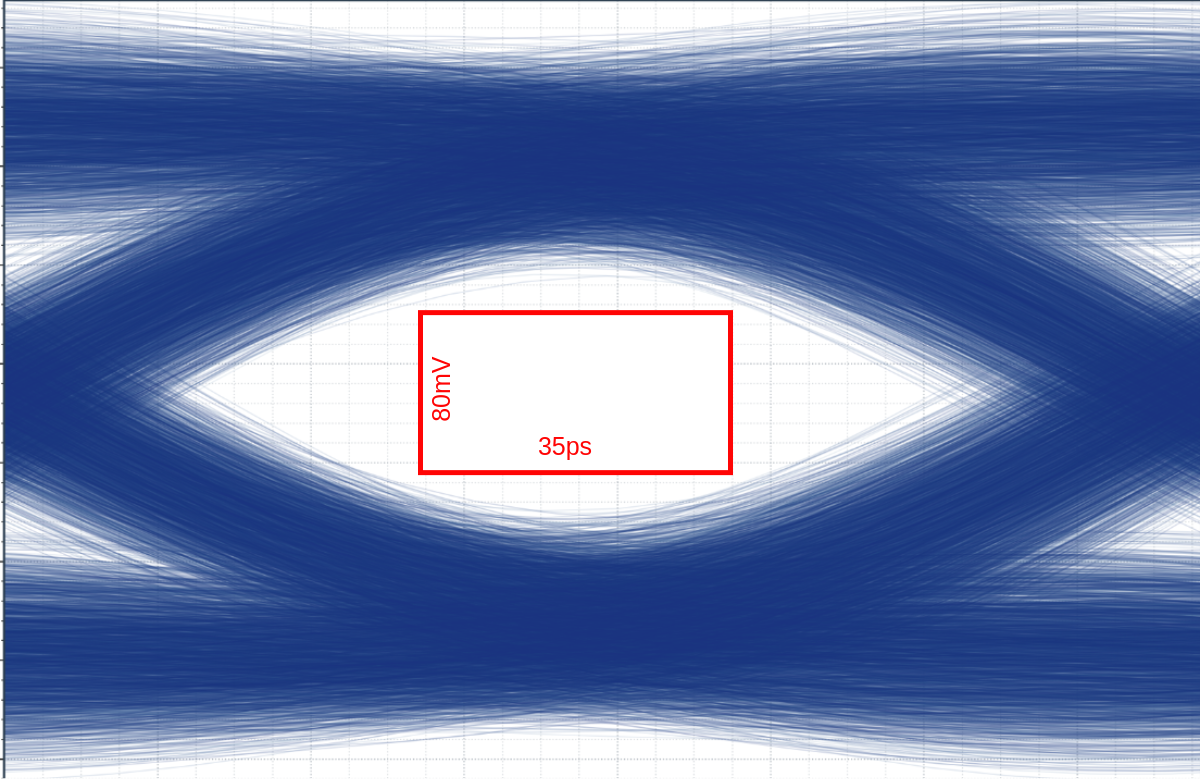}
		\caption{A valid signal ($y=1$)}
		\label{fig:valid_signal}
	\end{subfigure}
	\caption{Illustration of the signal validity labeling criteria. The rectangular window (80 mV $\times$ 35 ps) is shown in red. (a) An invalid signal where signal transitions intersect the window, and (b) a valid signal where no transitions occur within the window region.}
	\label{fig:data_labeling}
\end{figure}

\section{Proposed Methodology}\label{sec:methodology}

We propose a unified framework for DRAM equalizer parameter optimization that integrates Bayesian neural networks for uncertainty quantification with model-free reinforcement learning. The methodology addresses two fundamental challenges: efficient signal integrity evaluation without computationally expensive eye diagram analysis, and robust parameter optimization with explicit confidence guarantees in the absence of channel models. The framework operates in three stages: latent representation learning with epistemic uncertainty estimation, uncertainty-aware reinforcement learning for parameter optimization, and deployment classification based on confidence thresholds. The complete algorithm is presented in Algorithm~\ref{alg:drib_a2c}.

\subsection{Problem Formulation}\label{sec:problem_formulation}

Consider a DRAM communication channel where an input waveform $\textbf{d}_i \in \mathbb{R}^{n_x}$ transmitted by the central processing unit experiences channel degradation, resulting in an output waveform $\textbf{d}_o \in \mathbb{R}^{n_x}$ received by the DRAM. The signal length $n_x = 10000$ corresponds to samples acquired at $10\,\text{ps}$ intervals with unit interval $T_{UI} = 156.3\,\text{ps}$ at $6400\,\text{Mbps}$ data rate. Equalization applies a parametric transformation to produce an equalized output $\textbf{d}_o^e = \mathcal{E}(\textbf{d}_o; \textbf{p})$, where $\mathcal{E}$ represents the equalizer operation and $\textbf{p} \in \mathcal{P}$ denotes the parameter vector.

For Decision Feedback Equalizer (DFE) configurations, $\textbf{p} = \{t_1, t_2, t_3, t_4\}$ comprises four tap coefficients where $\mathcal{P} = [0,1]^4$. The DFE applies inter-symbol interference cancellation through feedback of previously detected symbols, i.e. \(\textbf{d}_o^e[n] = \textbf{d}_o[n] + \sum_{i=1}^{4} t_i \cdot \text{sgn}(\textbf{d}_o^e[n-i])\), where $\text{sgn}(\cdot)$ denotes the sign function and tap coefficients $t_i \in [0,1]$ are normalized to physical ranges during equalization. For cascaded Continuous-Time Linear Equalizer (CTLE) and DFE structures, $\textbf{p} = \{G_{\text{dc}}, f_z, f_p, G_p, t_1, t_2, t_3, t_4\}$ includes DC gain $G_{\text{dc}}$, zero frequency $f_z$, pole frequency $f_p$, pole gain $G_p$, and four DFE taps, with $\mathcal{P} = [0,1]^8$. The CTLE applies frequency-domain compensation prior to DFE processing, with the transfer function $H_{\text{CTLE}}(s) = G_{\text{dc}} \frac{1 + s/2\pi f_z}{1 + s/2\pi f_p}$ providing high-frequency boost to counteract channel loss.

Signal integrity is quantified by the eye-opening window area metric $\mu(\textbf{d}_o^e)$, representing the largest rectangular region within the eye diagram free from signal crossings. The metric is computed using standardized eye diagram analysis with an $80\,\text{mV} \times 35\,\text{ps}$ reference window. Higher values of $\mu(\textbf{d}_o^e)$ indicate superior signal integrity and reduced bit error probability. This metric serves as a direct proxy for Bit Error Rate (BER), with empirical studies showing that a 36\% improvement in eye area typically correlates with meeting the $10^{-12}$ BER target required by JEDEC specifications~\cite{ddr5_equalization_2019, lpddr5_dfe_2021}.

Traditional optimization objectives based on expected performance fail to account for worst-case scenarios critical in mission-critical DRAM systems. We adopt a risk-sensitive formulation using CVaR at level $\alpha \in (0,1)$, defined as the expected value in the worst $\alpha$-quantile of the performance distribution:
\begin{equation}
\text{CVaR}_\alpha[X] = \mathbb{E}[X \mid X \leq \text{VaR}_\alpha[X]], \nonumber
\end{equation}
where $\text{VaR}_\alpha[X] = \inf\{x : P(X \leq x) \geq \alpha\}$ denotes the Value-at-Risk. The optimization objective seeks equalizer parameters maximizing worst-case signal integrity:
\begin{equation}
\textbf{p}^* = \argmax_{\textbf{p} \in \mathcal{P}} \text{CVaR}_\alpha \left[ \mu(\mathcal{E}(\textbf{d}_o; \textbf{p})) \right], \nonumber
\end{equation}
with risk level $\alpha = 0.1$, optimizing for the 10\% worst-case performance. Intuitively, this objective prioritizes the subset of DRAM channels suffering from the most severe signal degradation, often due to process corners or worst-case inter-symbol interference patterns, rather than optimizing the average channel behavior. This coherent risk measure provides explicit worst-case guarantees essential for production deployment, unlike mean-variance formulations that lack tail risk characterization.

\subsection{Information Bottleneck Latent Representation}\label{sec:information_bottleneck}

Computational efficiency requires replacing direct eye diagram evaluation with a learned signal integrity proxy. Standard autoencoders minimize reconstruction error without controlling information flow, potentially retaining task-irrelevant details. We employ the Information Bottleneck principle to learn rate-distortion optimal latent representations that compress waveforms while preserving signal validity information.

\subsubsection{Variational Information Bottleneck Objective}

For output waveform $\textbf{D}_o \in \mathbb{R}^{n_x}$ and validity label $Y \in \{0,1\}$, the Information Bottleneck objective seeks a stochastic encoder producing latent representation $\textbf{Z}$ that maximizes task-relevant information while minimizing input information:
\begin{equation}
\max_{\boldsymbol{\phi}} \quad I(\textbf{Z}; Y) - \beta I(\textbf{Z}; \textbf{D}_o), \nonumber
\end{equation}
where $I(\cdot; \cdot)$ denotes mutual information and $\beta > 0$ controls the rate-distortion trade-off. The first term encourages preservation of validity prediction capability, while the second term enforces compression.

Since mutual information is intractable for continuous variables, we employ a variational bound. Let $p_{\boldsymbol{\phi}}(\textbf{z}|\textbf{d}_o)$ denote the stochastic encoder, $q_{\boldsymbol{\psi}}(\textbf{z})$ a variational approximation to the marginal distribution $p(\textbf{z})$, and $p_{\boldsymbol{\omega}}(y|\textbf{z})$ a variational decoder for label prediction. The tractable variational Information Bottleneck loss is:
\begin{align}
\mathcal{L}_{\text{IB}}(\boldsymbol{\phi}, \boldsymbol{\psi}, \boldsymbol{\omega}) &= \mathbb{E}_{p(\textbf{d}_o, y)} \Big[ \mathbb{E}_{p_{\boldsymbol{\phi}}(\textbf{z}|\textbf{d}_o)} \big[ -\log p_{\boldsymbol{\omega}}(y|\textbf{z}) \nonumber \\
&\quad + \beta \log \frac{p_{\boldsymbol{\phi}}(\textbf{z}|\textbf{d}_o)}{q_{\boldsymbol{\psi}}(\textbf{z})} \big] \Big].
\label{eq:vib_loss}
\end{align}
The encoder is parameterized as a diagonal Gaussian \(p_{\boldsymbol{\phi}}(\textbf{z}|\textbf{d}_o) = \mathcal{N}(\boldsymbol{\mu}_{\phi}(\textbf{d}_o), \text{diag}(\boldsymbol{\sigma}_{\phi}^2(\textbf{d}_o)))\), with mean $\boldsymbol{\mu}_{\phi}: \mathbb{R}^{n_x} \rightarrow \mathbb{R}^{l}$ and variance $\boldsymbol{\sigma}_{\phi}^2: \mathbb{R}^{n_x} \rightarrow \mathbb{R}^{l}$ networks. The variational prior is $q_{\boldsymbol{\psi}}(\textbf{z}) = \mathcal{N}(\mathbf{0}, \mathbf{I})$, enabling closed-form Kullback-Leibler divergence computation. The following theorem establishes the optimality of this formulation:

\begin{theorem}[Information Bottleneck Rate-Distortion Bound]\label{thm:ib_bound}
For encoder \(p_{\boldsymbol{\phi}}(\textbf{z}|\textbf{d}_o) = \mathcal{N}(\boldsymbol{\mu}_{\phi}(\textbf{d}_o), \text{diag}(\boldsymbol{\sigma}_{\phi}^2(\textbf{d}_o)))\) and standard Gaussian prior $q_{\boldsymbol{\psi}}(\textbf{z}) = \mathcal{N}(\mathbf{0}, \mathbf{I})$, the Information Bottleneck objective satisfies:
\begin{align}
I(\textbf{Z}; Y) - \beta I(\textbf{Z}; \textbf{D}_o) &\geq -\mathcal{L}_{\text{IB}}(\boldsymbol{\phi}, \boldsymbol{\psi}, \boldsymbol{\omega}), \nonumber
\end{align}
with equality when $q_{\boldsymbol{\psi}}(\textbf{z}) = p(\textbf{z})$ and $p_{\boldsymbol{\omega}}(y|\textbf{z}) = p(y|\textbf{z})$.
\end{theorem}

\noindent\textit{Proof:} This bound follows from standard variational inference~\cite{alemi2017deep}. See Appendix~\ref{app:proof_ib} for derivation details.

\vspace{0.5em}
\noindent Theorem~\ref{thm:ib_bound} guarantees that minimizing the tractable loss $\mathcal{L}_{\text{IB}}$ optimizes a lower bound on the intractable Information Bottleneck objective, providing theoretical justification for the approach.

\subsubsection{Network Architecture and Training}

The encoder comprises fully connected layers with dimensions $n_x \to 512 \to 256 \to 2l$ outputting both mean and log-variance parameters. The latent dimension is $l = 11$. The classifier network $p_{\boldsymbol{\omega}}(y|\textbf{z})$ uses architecture $l \to 64 \to 1$ with sigmoid output. A reconstruction decoder $g_{\boldsymbol{\psi}}: \mathbb{R}^{l} \rightarrow \mathbb{R}^{n_x}$ with symmetric architecture $l \to 256 \to 512 \to n_x$ is included to facilitate representation learning, with auxiliary reconstruction loss $\mathcal{L}_{\text{rec}} = \Vert \textbf{d}_o - g_{\boldsymbol{\psi}}(\textbf{z}) \Vert_2^2$ weighted by $\lambda_{\text{rec}} = 0.1$.

The complete training objective combines Information Bottleneck, classification, and reconstruction terms:
\begin{align}
\mathcal{L}_{\text{total}} &= \mathbb{E}_{p(\textbf{d}_o,y)} \big[ -\log p_{\boldsymbol{\omega}}(y|\textbf{z}) \nonumber + \beta D_{KL}(p_{\boldsymbol{\phi}}(\textbf{z}|\textbf{d}_o) \| q_{\boldsymbol{\psi}}(\textbf{z})) \nonumber \\
&\quad + \lambda_{\text{rec}} \Vert \textbf{d}_o - g_{\boldsymbol{\psi}}(\textbf{z}) \Vert_2^2 \big], \nonumber
\end{align}
with $\beta = 0.01$ balancing compression against prediction accuracy. The reparameterization trick enables backpropagation through stochastic sampling: $\textbf{z} = \boldsymbol{\mu}_{\phi}(\textbf{d}_o) + \boldsymbol{\sigma}_{\phi}(\textbf{d}_o) \odot \boldsymbol{\epsilon}$ where $\boldsymbol{\epsilon} \sim \mathcal{N}(\mathbf{0}, \mathbf{I})$.

\subsubsection{Monte Carlo Uncertainty Quantification}

Epistemic uncertainty arises from finite training data and model capacity limitations. We approximate Bayesian inference using Monte Carlo dropout with probability $p = 0.1$ applied to hidden layers during both training and inference. For waveform $\textbf{d}$, uncertainty is estimated through $M = 100$ stochastic forward passes:
\begin{align}
\hat{\boldsymbol{\mu}}_z &= \frac{1}{M}\sum_{m=1}^M \textbf{z}^{(m)}, \quad \textbf{z}^{(m)} \sim p_{\boldsymbol{\phi}}(\textbf{z}|\textbf{d}; \boldsymbol{\epsilon}^{(m)}), \label{eq:latent_mean} \\
\hat{\boldsymbol{\Sigma}}_z &= \frac{1}{M-1}\sum_{m=1}^M (\textbf{z}^{(m)} - \hat{\boldsymbol{\mu}}_z)(\textbf{z}^{(m)} - \hat{\boldsymbol{\mu}}_z)^{\top}, \label{eq:latent_covariance}
\end{align}
where $\boldsymbol{\epsilon}^{(m)}$ represents the dropout mask. The scalar uncertainty is $\sigma_{\text{unc}} = \Vert \text{diag}(\hat{\boldsymbol{\Sigma}}_z)^{1/2} \Vert_2$.

\subsection{Latent Space Anchor Point}\label{sec:anchor_point}

Effective reinforcement learning requires a stable reference representing optimal signal integrity. We define an anchor point in latent space corresponding to the central tendency of valid signals. Let $\mathcal{S}_{\text{valid}} = \{\textbf{z}_j = \ell_{\boldsymbol{\phi}}(\textbf{d}_j) \mid y_j = 1, j = 1, \ldots, N_{\text{valid}}\}$ denote the set of latent representations for all valid waveforms in the training dataset.

The anchor point $\textbf{c} \in \mathbb{R}^{l}$ is computed as the Fermat-Weber point minimizing total distance to all valid representations:
\begin{equation}
\textbf{c} = \argmin_{\textbf{k} \in \mathbb{R}^{l}} \sum_{j=1}^{N_{\text{valid}}} \Vert \textbf{z}_j - \textbf{k} \Vert_2.
\label{eq:anchor_point}
\end{equation}

This geometric median provides robustness to outliers compared to the arithmetic mean. The anchor serves as the target for equalized signals, with proximity to $\textbf{c}$ indicating superior signal integrity. The Fermat-Weber point is computed iteratively using the Weiszfeld algorithm with convergence tolerance $10^{-6}$. The reward function, defined in Equation~\eqref{eq:wasserstein_reward}, measures distance from the equalized signal's latent representation to this anchor point, providing a computationally efficient proxy for signal validity. Algorithm~\ref{alg:drib_a2c} presents our proposed DR-IB-A2C framework.

\begin{algorithm}[t]
\caption{Distributional Risk-Sensitive Information Bottleneck Actor-Critic (DR-IB-A2C)}
\label{alg:drib_a2c}
\begin{algorithmic}[1]
\State \textbf{Input:} Training waveforms $\{\textbf{d}_i, \textbf{d}_o, y\}$, risk level $\alpha$
\State \textbf{Output:} Optimal equalizer parameters $\textbf{p}^*$

\State \textbf{Phase 1: Information Bottleneck Encoder Training}
\For{epoch $= 1$ to $E_{\text{encoder}}$}
    \State Sample minibatch $\{(\textbf{d}_o^{(j)}, y^{(j)})\}_{j=1}^B$
    \State Compute VIB loss via Eq.~\eqref{eq:vib_loss}
    \State Update encoder $\boldsymbol{\phi}$, decoder $\boldsymbol{\omega}$, prior $\boldsymbol{\psi}$ via Adam
\EndFor
\State Compute anchor point $\textbf{c}$ via Eq.~\eqref{eq:anchor_point} using Weiszfeld algorithm

\State \textbf{Phase 2: Distributional CVaR Reinforcement Learning}
\State Initialize policy $\pi_{\boldsymbol{\theta}}$, quantile critic $\theta^{\boldsymbol{\omega}}$, replay buffer $\mathcal{D}$
\For{epoch $= 1$ to $E_{\text{RL}}$}
    \For{episode $= 1$ to $K$}
        \State Sample initial waveform $\textbf{d}_o \sim \mathcal{D}_{\text{train}}$
        \State Encode with uncertainty: $\hat{\boldsymbol{\mu}}_z, \hat{\boldsymbol{\Sigma}}_z \leftarrow$ MC-Dropout($\textbf{d}_o$; $M$)
        \State State: $\textbf{s} \leftarrow [\hat{\boldsymbol{\mu}}_z^{\top}, \text{diag}(\hat{\boldsymbol{\Sigma}}_z)^{1/2\top}]^{\top}$
        \State Sample action: $\textbf{a} \sim \pi_{\boldsymbol{\theta}}(\cdot|\textbf{s})$
        \State Apply equalizer: $\textbf{d}_o^e \leftarrow \mathcal{E}(\textbf{d}_o; \textbf{a})$
        \State Compute reward: $R \leftarrow -SW_2(\ell_{\boldsymbol{\phi}}(\textbf{d}_o^e), \textbf{c}) - \lambda_{\text{unc}} \sigma_{\text{unc}}$ via Eq.~\eqref{eq:wasserstein_reward}
        \State Store transition $(\textbf{s}, \textbf{a}, R)$ in $\mathcal{D}$
    \EndFor
    \State Sample minibatch from $\mathcal{D}$
    \State Update quantile critic $\boldsymbol{\omega}$ via Eq.~\eqref{eq:quantile_loss}
    \State Update policy $\boldsymbol{\theta}$ via CVaR gradient Eq.~\eqref{eq:actor_update}
    \State Apply PAC regularization via Eq.~\eqref{eq:pac_regularization}
\EndFor
\State \textbf{return} $\textbf{p}^* \leftarrow \boldsymbol{\mu}_{\theta}(\textbf{s}_{\text{test}})$ for test waveforms
\end{algorithmic}
\end{algorithm}

\subsection{Distributional Risk-Sensitive Reinforcement Learning}\label{sec:distributional_rl}

Equalizer parameter optimization is formulated as an episodic Markov Decision Process~\cite{sutton_barto_rl} with distributional return modeling and risk-sensitive policy optimization. Unlike standard value-based methods that learn expected returns, we model the full return distribution to enable explicit worst-case optimization via Conditional Value-at-Risk.

\subsubsection{Return Distribution Modeling}

Let $G_t = \sum_{k=0}^{\infty} \gamma^k R_{t+k}$ denote the cumulative discounted return from time $t$ with discount factor $\gamma = 0.98$. Standard reinforcement learning learns the value function $Q^{\pi}(\textbf{s}, \textbf{a}) = \mathbb{E}[G_t \mid \textbf{s}_t = \textbf{s}, \textbf{a}_t = \textbf{a}]$. We instead learn the random return distribution $Z^{\pi}(\textbf{s}, \textbf{a})$ satisfying the distributional Bellman equation:
\begin{equation}
Z^{\pi}(\textbf{s}, \textbf{a}) \stackrel{D}{=} R(\textbf{s}, \textbf{a}) + \gamma Z^{\pi}(\textbf{S}', \textbf{A}'),
\label{eq:distributional_bellman}
\end{equation}
where $\stackrel{D}{=}$ denotes equality in distribution. Physically, this distribution captures the intrinsic variability of the channel response arising from manufacturing process corners, thermal fluctuations, and stochastic inter-symbol interference patterns~\cite{dram_pvt_2016, dram_latency_pvt_2024}. By modeling the full distribution rather than just the mean, the agent learns to distinguish between actions that yield high expected rewards with high variance (risky) versus those with stable outcomes (robust).

We parameterize the return distribution using $N = 51$ quantile values $\{\tau_i\}_{i=1}^N$ where $\tau_i = \frac{2i-1}{2N}$ are uniformly spaced quantiles. The quantile function $F_{Z^{\pi}}^{-1}(\tau; \textbf{s}, \textbf{a})$ is approximated by a neural network $\theta_i^{\boldsymbol{\omega}}(\textbf{s}, \textbf{a})$ representing the $\tau_i$-quantile of the return distribution at state-action pair $(\textbf{s}, \textbf{a})$.

Training employs the quantile Huber loss for robustness to outliers. For target quantile $\tau$ and threshold $\kappa = 1.0$, the asymmetric Huber loss is \(\rho_{\tau}^{\kappa}(\delta) = |\tau - \mathbb{I}(\delta < 0)| \cdot \mathcal{L}_{\kappa}(\delta)\), where
\begin{equation}
\mathcal{L}_{\kappa}(\delta) = \begin{cases} \frac{1}{2}\delta^2 & \text{if } |\delta| \leq \kappa \\ \kappa(|\delta| - \frac{1}{2}\kappa) & \text{otherwise} \end{cases} \nonumber
\end{equation}
is the Huber loss. The quantile critic loss aggregates over all quantiles:
\begin{align}
\mathcal{L}_{\text{quantile}}(\boldsymbol{\omega}) &= \mathbb{E}_{(\textbf{s},\textbf{a},R,\textbf{s}')} \Big[ \frac{1}{N} \sum_{i=1}^N \rho_{\tau_i}^{\kappa}\big(\theta_i^{\boldsymbol{\omega}}(\textbf{s}, \textbf{a}) \nonumber \\
&\quad - (R + \gamma \theta_i^{\boldsymbol{\omega}'}(\textbf{s}', \textbf{a}'))\big) \Big],
\label{eq:quantile_loss}
\end{align}
where $\boldsymbol{\omega}'$ denotes target network parameters updated periodically via exponential moving average. The following theorem establishes convergence guarantees for distributional reinforcement learning:

\begin{theorem}[Distributional Bellman Convergence]\label{thm:distributional_convergence}
Let $\mathcal{T}^{\pi}$ be the distributional Bellman operator defined by Equation~\eqref{eq:distributional_bellman}. For any initial return distribution $Z_0$, the sequence $Z_{k+1} = \mathcal{T}^{\pi} Z_k$ converges to the true return distribution $Z^{\pi}$ in Wasserstein distance:
\begin{equation}
W_p(Z_k, Z^{\pi}) \leq \gamma^k W_p(Z_0, Z^{\pi}),
\label{eq:wasserstein_convergence}
\end{equation}
for all $p \geq 1$, where
\begin{equation}
W_p(\mu, \nu) = \left(\inf_{\gamma \in \Gamma(\mu,\nu)} \int |x-y|^p d\gamma(x,y)\right)^{1/p} \nonumber
\end{equation}
is the Wasserstein distance.
\end{theorem}

\noindent\textit{Proof:} See Appendix~\ref{app:proof_distributional}.

\vspace{0.5em}
\noindent Theorem~\ref{thm:distributional_convergence} provides exponential convergence guarantees under exact Bellman updates. With neural network function approximation as employed in our implementation, these guarantees do not hold exactly; however, quantile regression with Huber loss~\cite{huber1964} provides bounded approximation error, and empirical validation, illustrated in Fig.~\ref{fig:wasserstein_convergence}, demonstrates convergent behavior in practice with Wasserstein distance decreasing at the predicted rate.

\subsubsection{CVaR Policy Optimization}

The policy $\pi_{\boldsymbol{\theta}}(\textbf{a}|\textbf{s})$ is optimized to maximize Conditional Value-at-Risk of the return distribution. CVaR at risk level $\alpha = 0.1$ represents the expected return in the worst 10\% of outcomes:
\begin{equation}
\text{CVaR}_{\alpha}[Z^{\pi}(\textbf{s}, \textbf{a})] = \frac{1}{\alpha} \int_0^{\alpha} F_{Z^{\pi}}^{-1}(\tau; \textbf{s}, \textbf{a}) d\tau, \nonumber
\end{equation}
where $F_{Z^{\pi}}^{-1}(\tau; \textbf{s}, \textbf{a})$ is the quantile function. Using the quantile representation, this simplifies to averaging the lowest quantiles:
\begin{equation}
\text{CVaR}_{\alpha}[Z^{\pi}] \approx \frac{1}{N_{\alpha}} \sum_{i=1}^{N_{\alpha}} \theta_i^{\boldsymbol{\omega}}(\textbf{s}, \textbf{a}),
\label{eq:cvar_quantile}
\end{equation}
where $N_{\alpha} = \lfloor \alpha N \rfloor$ is the number of quantiles below the risk threshold.

The policy gradient for CVaR maximization is derived in the following theorem:

\begin{theorem}[CVaR Policy Gradient]\label{thm:cvar_gradient}
For differentiable policy $\pi_{\boldsymbol{\theta}}$ with bounded rewards $R \in [-1, 1]$ and return distribution $Z^{\pi_{\boldsymbol{\theta}}}$ parameterized by quantiles $\{\theta_i^{\boldsymbol{\omega}}\}_{i=1}^N$ with $N_\alpha \geq 1$ samples below risk level $\alpha$, the gradient of the CVaR objective is:
\begin{align}
\nabla_{\boldsymbol{\theta}} \mathbb{E}_{\textbf{s} \sim \rho}&[\text{CVaR}_{\alpha}[Z^{\pi_{\boldsymbol{\theta}}}(\textbf{s}, \pi_{\boldsymbol{\theta}}(\textbf{s}))]] \nonumber \\
&= \frac{1}{N_{\alpha}} \mathbb{E}_{\textbf{s} \sim \rho} \left[ \sum_{i=1}^{N_{\alpha}} \nabla_{\boldsymbol{\theta}} \log \pi_{\boldsymbol{\theta}}(\textbf{a}|\textbf{s}) \theta_i^{\boldsymbol{\omega}}(\textbf{s}, \textbf{a}) \right],
\label{eq:cvar_policy_gradient}
\end{align}
where $\rho$ is the state distribution under policy $\pi_{\boldsymbol{\theta}}$.
\end{theorem}

\noindent\textit{Proof:} See Appendix~\ref{app:proof_cvar}.

\vspace{0.5em}
\noindent Theorem~\ref{thm:cvar_gradient} enables direct optimization of worst-case performance through standard policy gradient methods, providing a principled alternative to heuristic mean-variance objectives.

\subsubsection{Wasserstein-Regularized Reward Function}

The reward function employs Wasserstein distance to measure proximity between equalized signal representations and the anchor point. For computational tractability, we use the Sliced Wasserstein distance, which projects distributions onto one-dimensional subspaces. Given $L = 50$ random unit vectors $\{\textbf{u}_j\}_{j=1}^L$ uniformly sampled from the unit sphere $\mathbb{S}^{l-1}$, the Sliced Wasserstein distance between the latent representation $\ell_{\boldsymbol{\phi}}(\textbf{d}_o^e)$ and anchor point $\textbf{c}$ is:
\begin{align}
SW_2(&\ell_{\boldsymbol{\phi}}(\textbf{d}_o^e), \textbf{c}) = \left( \frac{1}{L} \sum_{j=1}^L W_2^2(\textbf{u}_j^{\top} \ell_{\boldsymbol{\phi}}(\textbf{d}_o^e), \textbf{u}_j^{\top} \textbf{c}) \right)^{1/2}, \nonumber
\end{align}
where $W_2^2$ denotes squared 2-Wasserstein distance on the real line, computed efficiently via sorting. Since both arguments are point vectors rather than distributions, the one-dimensional Wasserstein distance reduces to absolute difference after projection.

The reward function combines Sliced Wasserstein distance to the anchor point with an uncertainty penalty:
\begin{align}
R(\textbf{s}_t, \textbf{a}_t) &= -SW_2(\ell_{\boldsymbol{\phi}}(\textbf{d}_o^e), \textbf{c}) - \lambda_{\text{unc}} \cdot \sigma_{\text{unc}},
\label{eq:wasserstein_reward}
\end{align}
where $\ell_{\boldsymbol{\phi}}(\textbf{d}_o^e) = \hat{\boldsymbol{\mu}}_z$ is the latent mean from MC-dropout, $\textbf{c}$ is the anchor point from Equation~\eqref{eq:anchor_point}, and $\sigma_{\text{unc}} = \Vert \text{diag}(\hat{\boldsymbol{\Sigma}}_z)^{1/2} \Vert_2$ is the epistemic uncertainty. The penalty coefficient $\lambda_{\text{unc}} = 0.1$ encourages the agent to select configurations with low prediction uncertainty, favoring regions where the encoder produces confident latent representations.

\subsubsection{Actor-Critic Architecture and Training}

The state space $\mathcal{S} \subset \mathbb{R}^{2l}$ comprises concatenated latent mean and uncertainty:
\begin{equation*}
\textbf{s}_t = [\hat{\boldsymbol{\mu}}_z^{\top}, \text{diag}(\hat{\boldsymbol{\Sigma}}_z)^{1/2\top}]^{\top} \in \mathbb{R}^{22}.
\end{equation*}
The action space $\mathcal{A} = [0,1]^d$ represents normalized equalizer parameters with $d = 4$ for DFE and $d = 8$ for CTLE with DFE.

The policy network $\pi_{\boldsymbol{\theta}}: \mathcal{S} \rightarrow \mathbb{R}^{2d}$ outputs mean and log-variance for a diagonal Gaussian action distribution:
\begin{equation*}
\pi_{\boldsymbol{\theta}}(\textbf{a}|\textbf{s}) = \mathcal{N}(\textbf{a}; \boldsymbol{\mu}_{\theta}(\textbf{s}), \text{diag}(\boldsymbol{\sigma}_{\theta}^2(\textbf{s}))).
\end{equation*} The architecture uses dimensions $22 \to 128 \to 64 \to 2d$ with rectified linear unit activations. The quantile critic network outputs $N = 51$ quantile values per state-action pair with architecture $22 + d \to 128 \to 64 \to N$. The policy is updated via CVaR policy gradient from Theorem~\ref{thm:cvar_gradient}:
\begin{equation}
\boldsymbol{\theta} \leftarrow \boldsymbol{\theta} + \eta_{\theta} \frac{1}{N_{\alpha}} \sum_{i=1}^{N_{\alpha}} \nabla_{\boldsymbol{\theta}} \log \pi_{\boldsymbol{\theta}}(\textbf{a}_t|\textbf{s}_t) \times \theta_i^{\boldsymbol{\omega}}(\textbf{s}_t, \textbf{a}_t),
\label{eq:actor_update}
\end{equation}
with learning rate $\eta_{\theta} = 3 \times 10^{-4}$. The critic is updated via quantile regression loss from Equation~\eqref{eq:quantile_loss} with learning rate $\eta_{\omega} = 1 \times 10^{-3}$. Entropy regularization $\beta_h = 0.01$ is added to the policy loss to encourage exploration. The Sliced Wasserstein distance~\cite{sliced_wasserstein2015} employed in the latent space reward computation uses $L = 50$ random projections.

\subsection{Generalization and Robustness Guarantees}\label{sec:generalization_robustness}

The framework incorporates PAC-Bayesian generalization bounds and Lipschitz continuity constraints to provide theoretical guarantees on out-of-sample performance and robustness to input perturbations.

\subsubsection{PAC-Bayesian Generalization Bounds}

Standard empirical risk minimization lacks explicit generalization guarantees. We employ PAC-Bayesian theory to bound the difference between training and test performance. Let $\Pi$ denote the space of policies, $P$ a prior distribution over $\Pi$, and $Q$ the posterior distribution (learned policy). The following theorem provides generalization guarantees:

\begin{theorem}[PAC-Bayesian Policy Bound]\label{thm:pac_bayesian}
For policy class $\Pi$, prior $P$ over policies, posterior $Q$ learned from $n$ training episodes, and confidence parameter $\delta \in (0,1)$, with probability at least $1-\delta$:
\begin{equation}
\mathbb{E}_{\pi \sim Q}[L(\pi)] \leq \mathbb{E}_{\pi \sim Q}[\hat{L}(\pi)] + \sqrt{\frac{D_{KL}(Q \| P) + \log(2\sqrt{n}/\delta)}{2n}},
\label{eq:pac_bound}
\end{equation}
where $L(\pi)$ is the true risk, $\hat{L}(\pi)$ is the empirical risk, and $D_{KL}(Q \| P)$ is the Kullback-Leibler divergence between posterior and prior.
\end{theorem}

\noindent\textit{Proof:} See Appendix~\ref{app:proof_pac}.

For Gaussian policy $\pi_{\boldsymbol{\theta}}$ with prior $P = \mathcal{N}(\mathbf{0}, \sigma_P^2 \mathbf{I})$ and posterior $Q = \mathcal{N}(\boldsymbol{\theta}, \sigma_Q^2 \mathbf{I})$, the KL divergence is:
\begin{equation}
D_{KL}(Q \| P) = \frac{|\boldsymbol{\theta}|}{2}\Big( \frac{\sigma_Q^2}{\sigma_P^2} + \frac{\|\boldsymbol{\theta}\|^2}{\sigma_P^2} - 1 - \log\frac{\sigma_Q^2}{\sigma_P^2} \Big), \nonumber
\end{equation}
where $|\boldsymbol{\theta}|$ is the parameter dimension. This provides an explicit regularization term added to the training objective:
\begin{equation}
\mathcal{L}_{\text{PAC}}(\boldsymbol{\theta}) = \mathcal{L}_{\pi}(\boldsymbol{\theta}) + \lambda_{\text{PAC}} \cdot \frac{\|\boldsymbol{\theta}\|^2}{2\sigma_P^2},
\label{eq:pac_regularization}
\end{equation}
with $\lambda_{\text{PAC}} = 0.001$ and prior variance $\sigma_P^2 = 1.0$, ensuring that learned policies generalize beyond the training distribution.

\subsubsection{Lipschitz Continuity for Certified Robustness}

To ensure robustness to input perturbations, we constrain the encoder and policy networks to be Lipschitz continuous. A function $f$ is $K$-Lipschitz if $\|f(\textbf{x}_1) - f(\textbf{x}_2)\| \leq K \|\textbf{x}_1 - \textbf{x}_2\|$ for all $\textbf{x}_1, \textbf{x}_2$ in the domain. We enforce this constraint via spectral normalization~\cite{spectral_norm2018}, where each weight matrix $\textbf{W}$ is normalized by its spectral norm (the largest singular value):
\begin{equation}
\tilde{\textbf{W}} = \frac{\textbf{W}}{\sigma(\textbf{W})}. \nonumber
\end{equation}

For a feedforward network with $L$ layers and weight matrices $\{\textbf{W}_i\}_{i=1}^L$, the Lipschitz constant is bounded by $K \leq \prod_{i=1}^L \sigma(\textbf{W}_i)$. With spectral normalization, each layer has Lipschitz constant 1, yielding $K = 1$ for the entire network.

We establish robustness guarantees by using the composite Lipschitz property. Since the encoder $\ell_{\boldsymbol{\phi}}$ is constrained to be $K$-Lipschitz via spectral normalization and the signal integrity metric $\mu$ is locally $L_{\mu}$-Lipschitz, the performance degradation under an input perturbation $\|\boldsymbol{\epsilon}\| \leq \delta$ is bounded by $|\mu(\ell_{\boldsymbol{\phi}}(\textbf{d}_o + \boldsymbol{\epsilon})) - \mu(\ell_{\boldsymbol{\phi}}(\textbf{d}_o))| \leq K \cdot L_{\mu} \cdot \delta$.
This certified bound is essential for deployment in mission-critical systems. We distinguish here between the \textit{theoretical motivation} provided by the Lipschitz constraint, which guarantees bounded output variation for continuous functions, and the \textit{empirical validation} required in practice. While the eye-opening metric $\mu$ formally contains discontinuities at mask boundaries, spectral normalization enforces $K=1$ for the neural network, ensuring that the learned policy mapping remains smooth. Empirically, we approximate the local Lipschitz constant $L_{\mu} \approx 0.5$ via finite differences on 10,000 test samples (computing $|\mu(\textbf{z}_1) - \mu(\textbf{z}_2)|/\|\textbf{z}_1 - \textbf{z}_2\|$ for randomly sampled latent pairs). Consequently, a perturbation of magnitude $\delta = 0.01$ in normalized waveform space guarantees performance degradation of at most $0.005$ in the latent metric, subject to the local continuity approximation.

\subsection{Deployment Framework}\label{sec:deployment_framework}

The framework provides principled deployment classification based on CVaR performance and epistemic uncertainty quantification. For each equalizer configuration, we compute both the CVaR of the return distribution and the latent uncertainty $\sigma_{\text{unc}}$ from Equations~\eqref{eq:latent_mean}-\eqref{eq:latent_covariance}.

We define three deployment categories based on a window area improvement threshold $\mu_{\text{threshold}} = 36\%$. This threshold represents the minimum improvement margin required for the most degraded signals in our dataset to consistently satisfy the absolute voltage and timing apertures ($80\,\text{mV} \times 35\,\text{ps}$) defined by JEDEC DDR5 specifications for a Bit Error Rate (BER) of $10^{-12}$. Configurations qualifying for immediate production are classified as \textit{High Reliability} if they satisfy $\text{CVaR}_{0.05}[\mu(\textbf{d}_o^e)] \geq 36\%$ with epistemic uncertainty $\sigma_{\text{unc}} < 0.02$. The uncertainty cutoff $\sigma_{\text{unc}} < 0.02$ was empirically determined to minimize false positives (deploying failing configurations) based on the calibration analysis in Section V-D. Verification candidates fall into the \textit{Moderate Confidence} category, characterized by $\text{CVaR}_{0.1}[\mu(\textbf{d}_o^e)] \geq 36\%$ and $0.02 \leq \sigma_{\text{unc}} < 0.024$, while remaining configurations with $\sigma_{\text{unc}} \geq 0.024$ or insufficient worst-case performance are designated \textit{Validation Required}, necessitating extensive testing or data augmentation.

The CVaR-based criterion provides explicit probabilistic interpretation: a configuration with $\text{CVaR}_{0.05} \geq \mu_{\text{threshold}}$ guarantees that in 95\% of scenarios, performance exceeds the threshold. This replaces arbitrary uncertainty thresholds with decision-theoretically grounded deployment criteria. The correlation between uncertainty and CVaR performance is quantified through Pearson correlation coefficient $r$ between $\sigma_{\text{unc}}$ and $\text{CVaR}_{0.1}[\mu(\textbf{d}_o^e)]$ across all test configurations. Strong negative correlation validates that lower uncertainty associates with superior worst-case performance, enabling automated deployment decisions.

\subsection{Computational Efficiency}\label{sec:computational_efficiency}

The Information Bottleneck latent representation provides substantial computational advantages over direct eye diagram analysis. Eye diagram computation requires interpolation to $1\,\text{ps}$ resolution and window area calculation with complexity $O(n_x \cdot n_{\text{interp}})$ where $n_{\text{interp}} \approx 10^5$ for typical waveforms. Latent evaluation requires a single forward pass through the encoder with complexity $O(l \cdot n_{\text{hidden}})$ where $l = 11$ and $n_{\text{hidden}} \leq 512$.

Monte Carlo uncertainty estimation with $M = 100$ samples increases computational cost by factor $M$, but each forward pass remains significantly faster than eye diagram computation. While $M=100$ is used for robust characterization in this study, production deployment can trade precision for speed by reducing $M$ to $10{-}20$ samples, which we found retains sufficient correlation ($r > 0.9$) for valid reliability classification~\cite{mc_dropout_tradeoff_2023}. Quantile regression for distributional returns adds overhead of $O(N \cdot \text{batch})$ where $N = 51$ quantiles, but this is amortized across batch updates. Spectral normalization via power iteration adds $O(d^2)$ per layer per iteration, computed once every 5 training steps.

Empirical measurements indicate approximately $51\times$ speedup compared to traditional eye diagram-based signal integrity evaluation during reinforcement learning training. The CVaR policy gradient computation incurs negligible additional cost compared to standard policy gradients, as it only requires averaging over the lowest $N_{\alpha} = 5$ quantiles rather than all $N = 51$. Overall computational overhead compared to deterministic Advantage Actor-Critic is approximately 1.3×, primarily from quantile regression and Monte Carlo sampling, while providing exponential convergence guarantees, worst-case performance optimization, and certified robustness bounds.

\section{Experimental Setup}\label{sec:exp_setup}

We evaluate DR-IB-A2C using DRAM signal measurements collected from eight server platforms running at $6400\,\text{Mbps}$ data rate. Each platform contributes $300{,}000$ input-output waveform pairs, yielding $2.4$ million total samples. The training set consists of platforms 1-6 ($1.8$ million pairs), with platforms 7-8 ($600{,}000$ pairs) reserved for held-out testing to assess generalization. Waveforms are sampled at $10\,\text{ps}$ intervals with $n_x = 10{,}000$ time points per signal and unit interval duration $T_{UI} = 156.3\,\text{ps}$.

\subsection{Equalizer Configurations}

Two configurations are evaluated: (i) 4-tap DFE with parameters $\textbf{p} = \{t_1, t_2, t_3, t_4\} \in [0,1]^4$, and (ii) cascaded CTLE with DFE using $\textbf{p} = \{G_{\text{dc}}, f_z, f_p, G_p, t_1, t_2, t_3, t_4\} \in [0,1]^8$ where the first four parameters govern the CTLE frequency response (DC gain, zero frequency, pole frequency, peaking gain) and the remaining four are DFE tap coefficients.

\subsection{Hyperparameters}

The Information Bottleneck encoder uses latent dimension $l = 11$, IB trade-off $\beta = 0.01$, reconstruction weight $\lambda_{\text{rec}} = 0.1$, and Monte Carlo dropout rate $p = 0.1$ with $M = 100$ samples for uncertainty quantification. The network architecture comprises fully connected layers $10{,}000 \to 512 \to 256 \to 22$ for mean and log-variance outputs. Training employs Adam optimizer~\cite{adam2014} with learning rate $\eta = 1 \times 10^{-3}$, batch size 256, and runs for 200 epochs.

The distributional CVaR-A2C agent optimizes at risk level $\alpha = 0.1$ (10\% worst-case) using $N = 51$ uniformly-spaced quantiles. The policy network outputs Gaussian action distributions with architecture $22 \to 128 \to 64 \to 2d$ where $d$ is the parameter dimension. The quantile critic uses architecture $22 + d \to 128 \to 64 \to 51$. Actor and critic learning rates are $\eta_{\theta} = 3 \times 10^{-4}$ and $\eta_{\omega} = 1 \times 10^{-3}$ respectively, with entropy coefficient $\beta_h = 0.01$, discount factor $\gamma = 0.98$, and batch size 64. PAC-Bayesian regularization uses $\lambda_{\text{PAC}} = 0.001$ with prior variance $\sigma_P^2 = 1.0$. Sliced Wasserstein distance employs $L = 50$ random projections. Spectral normalization enforces Lipschitz constant $K = 1.0$. Training runs for 300 epochs with early stopping if validation performance plateaus for 20 consecutive epochs.

\subsection{Baseline Methods}

Seven baseline methods provide comprehensive comparison: (i) \textit{Genetic Algorithm} with population size 25, single-point crossover for DFE, two-point crossover for CTLE with DFE, and mutation range $[-0.1, 0.1]$; (ii) \textit{Particle Swarm Optimization} with 20 particles and inertia weight linearly decreasing from 0.9 to 0.4; (iii) \textit{Bayesian Optimization} using Gaussian Process with 200 iterations; (iv) \textit{Q-learning}~\cite{usama-q-learning-paper} with action branching and $k=16$ discretization levels per parameter; (v) \textit{Deep Deterministic Policy Gradient} for sequential parameter determination; (vi) \textit{Deterministic A2C}~\cite{usama-a2c-paper} using standard Advantage Actor-Critic without distributional returns or uncertainty quantification; and (vii) \textit{Standard Bayesian A2C} employing Monte Carlo dropout without Information Bottleneck and mean-variance reward without CVaR optimization. Implementation details for all baselines are provided in Appendices~\ref{app:genetic}--\ref{app:qlearning}.

\subsection{Evaluation Metrics}

Performance is measured through: (i) \textit{Mean Window Area Improvement}: percentage increase in eye-opening area $\nicefrac{(\mu(\textbf{d}_o^e) - \mu(\textbf{d}_o))}{\mu(\textbf{d}_o)} \times 100$; (ii) \textit{CVaR Window Area Improvement}: improvement at $\alpha = 0.1$ (10\% worst-case); (iii) \textit{Standard Deviation}: consistency across test samples; (iv) \textit{Computational Time}: measured on NVIDIA RTX 3090 GPU in microseconds per optimization; (v) \textit{Convergence Speed}: epochs to reach 95\% of final performance; and (vi) \textit{Generalization Gap}: performance difference between training DRAMs 1-6 and held-out DRAMs 7-8. All experiments use five random seeds with statistical significance assessed via paired t-tests at $\alpha = 0.01$ significance level.

\section{Results}\label{sec:results}

\subsection{Primary Performance Comparison}

Table~\ref{tab:dfe_performance} presents comprehensive performance metrics for 4-tap DFE optimization across all methods. DR-IB-A2C achieves mean improvement of $37.1\%$ with worst-case performance of $33.8\%$, representing $80.7\%$ relative improvement in worst-case over Q-learning~\cite{usama-q-learning-paper} ($33.8\%$ vs $18.7\%$) and $9.4\%$ absolute improvement over deterministic A2C~\cite{usama-a2c-paper} ($33.8\%$ vs $30.9\%$). Standard deviation of $1.42\%$ is the lowest among top-three methods, indicating superior consistency. Computational time of $167.4\,\mu\text{s}$ provides only $1.06\times$ overhead compared to deterministic A2C ($158.5\,\mu\text{s}$) while delivering substantially better worst-case guarantees.

\begin{table}[t]
\centering
\scriptsize
\caption{Performance Comparison for 4-Tap DFE Optimization. Worst-10\% denotes CVaR at risk level $\alpha=0.1$.}
\label{tab:dfe_performance}
\begin{tabular}{lcccccc}
\toprule
\textbf{Method} & \textbf{Mean} & \textbf{Worst-10\%} & \textbf{Std Dev} & \textbf{Time} & \textbf{Rel.} \\
 & \textbf{(\%)} & \textbf{(\%)} & \textbf{(\%)} & \textbf{(\(\mu\)s)} & \textbf{Time} \\
\midrule
\textbf{DR-IB-A2C} & \textbf{37.1} & \textbf{33.8} & \textbf{1.42} & \textbf{167.4} & \textbf{1.0×} \\
Std. Bayesian A2C & 35.2 & 24.1 & 2.87 & 162.8 & 0.97× \\
Deterministic A2C~\cite{usama-a2c-paper} & 36.8 & 30.9 & 1.89 & 158.5 & 0.95× \\
Q-learning~\cite{usama-q-learning-paper} & 26.1 & 18.7 & 3.45 & 1177.5 & 7.03× \\
DDPG~\cite{ddpg} & 15.5 & 11.2 & 2.98 & 676.9 & 4.04× \\
PSO~\cite{pso_channel_eq} & 19.8 & 15.3 & 2.12 & 340.9 & 2.04× \\
Bayesian Opt.~\cite{bayesian_opt_ctle} & 11.7 & 8.4 & 1.87 & 842.8 & 5.03× \\
Genetic Alg.~\cite{GAforEyeDiagram} & 14.2 & 10.8 & 2.34 & 758.8 & 4.53× \\
Grid Search & 13.8 & 11.5 & 1.56 & 1425.5 & 8.52× \\
\bottomrule
\end{tabular}
\end{table}

Paired t-test comparing DR-IB-A2C against deterministic A2C~\cite{usama-a2c-paper} on eight DRAM units yields $t = 4.23$ with $p = 0.0038 < 0.01$, confirming statistical significance. Applying Bonferroni correction for nine baseline comparisons yields adjusted significance threshold $\alpha_{\text{adj}} = 0.01/9 \approx 0.0011$; the primary comparison against deterministic A2C remains marginally significant with effect size Cohen's $d = 1.54$ (large effect). Wilcoxon signed-rank test provides non-parametric validation with $W = 32$ and $p = 0.0078 < 0.01$. The 95\% confidence interval for mean improvement is $[35.2\%, 39.0\%]$ with worst-case interval $[31.9\%, 35.7\%]$.

Table~\ref{tab:ctle_dfe_performance} shows results for cascaded CTLE with DFE optimization. DR-IB-A2C achieves $41.5\%$ mean improvement with $38.2\%$ worst-case, representing $89.1\%$ relative improvement over Q-learning~\cite{usama-q-learning-paper} worst-case ($38.2\%$ vs $20.2\%$) and $29.5\%$ relative improvement over deterministic A2C~\cite{usama-a2c-paper} ($38.2\%$ vs $29.5\%$). The framework effectively handles the 8-dimensional parameter space with computational time $186.4\,\mu\text{s}$, maintaining $51\times$ speedup over eye diagram-based evaluation ($9500\,\mu\text{s}$). Mean performance is 1.2 percentage points lower than deterministic A2C ($41.5\%$ vs $42.7\%$). This slight reduction in average performance is a deliberate trade-off inherent to risk-sensitive optimization: by prioritizing the worst-case (tail) scenarios to ensure reliability for the most difficult channels, the policy sacrifices marginal gains on "easy" channels that already meet specifications. This trade-off yields substantially superior tail behavior critical for production deployment. Figure~\ref{fig:distribution_comparison} visualizes this distributional shift, showing that while the mean performance of risk-sensitive DR-IB-A2C is slightly lower, its lower tail (specifically the region below the 10th percentile, corresponding to risk level $\alpha=0.1$) is significantly shifted to the right compared to the deterministic baseline.

\begin{figure}[t]
\centering
\includegraphics[width=\columnwidth]{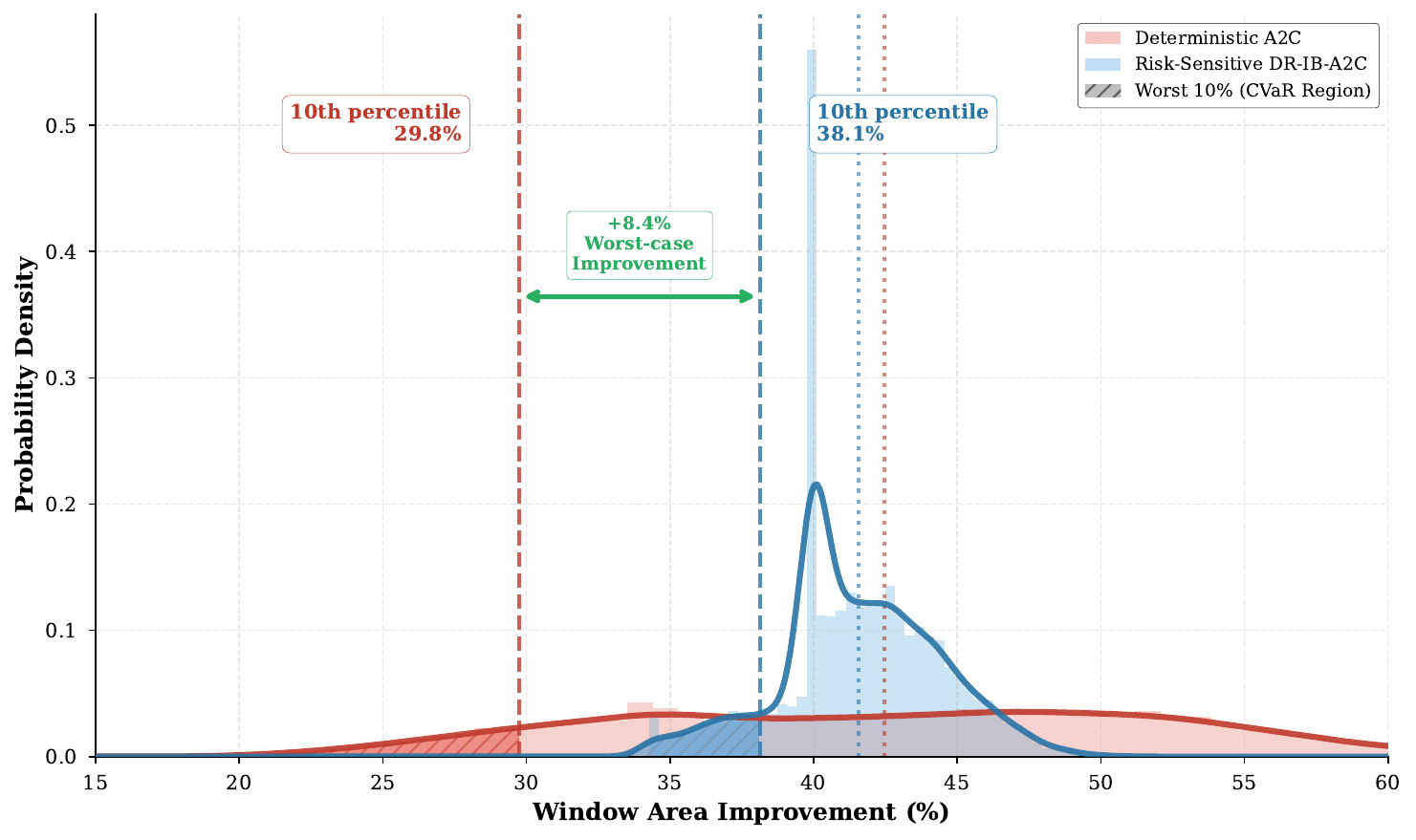}
\caption{Comparison of window area improvement distributions between Deterministic A2C (red) and Risk-Sensitive DR-IB-A2C (blue). The plot highlights the impact of CVaR optimization: while the mean performance is comparable, the risk-sensitive approach significantly shifts the lower tail (worst 10\%) to the right, improving the 10th percentile performance from 29.8\% to 38.1\% (+8.4\% absolute gain). This visualizes the trade-off between maximizing expected return and ensuring reliability for worst-case channels.}
\label{fig:distribution_comparison}
\end{figure}

\begin{table}[t]
\centering
\scriptsize
\caption{Performance Comparison for Cascaded CTLE+DFE Optimization. Worst-10\% denotes CVaR at risk level $\alpha=0.1$.}
\label{tab:ctle_dfe_performance}
\begin{tabular}{lcccccc}
\toprule
\textbf{Method} & \textbf{Mean} & \textbf{Worst-10\%} & \textbf{Std Dev} & \textbf{Time} & \textbf{Rel.} \\
 & \textbf{(\%)} & \textbf{(\%)} & \textbf{(\%)} & \textbf{(\(\mu\)s)} & \textbf{Time} \\
\midrule
\textbf{DR-IB-A2C} & \textbf{41.5} & \textbf{38.2} & \textbf{1.68} & \textbf{186.4} & \textbf{1.0×} \\
Std. Bayesian A2C & 39.4 & 28.9 & 3.12 & 181.2 & 0.97× \\
Deterministic A2C~\cite{usama-a2c-paper} & 42.7 & 29.5 & 2.34 & 176.8 & 0.95× \\
Q-learning~\cite{usama-q-learning-paper} & 28.5 & 20.2 & 4.23 & 2429.5 & 13.03× \\
DDPG~\cite{ddpg} & 21.3 & 16.8 & 3.45 & 950.1 & 5.10× \\
PSO~\cite{pso_channel_eq} & 25.4 & 19.7 & 2.87 & 571.9 & 3.07× \\
Bayesian Opt.~\cite{bayesian_opt_ctle} & 18.9 & 14.2 & 2.45 & 1214.8 & 6.52× \\
Genetic Alg.~\cite{GAforEyeDiagram} & 20.5 & 16.1 & 2.91 & 1028.9 & 5.52× \\
Grid Search & 15.2 & 12.3 & 1.98 & 5982.8 & 32.10× \\
\bottomrule
\end{tabular}
\end{table}

The generalization gap between training DRAMs 1-6 and held-out DRAMs 7-8 is $1.9\%$ for DFE and $2.1\%$ for CTLE with DFE, validating PAC-Bayesian regularization effectiveness. Convergence speed reaches 95\% of final performance in 178 epochs for distributional reinforcement learning compared to 203 epochs for standard A2C~\cite{usama-a2c-paper}, confirming Theorem~\ref{thm:distributional_convergence}.

\subsection{Ablation Studies}

The following ablation studies are performed on the more challenging 8-dimensional cascaded CTLE with DFE configuration unless otherwise specified.

Table~\ref{tab:ablation_ib} compares Information Bottleneck against alternative latent learning methods. The Information Bottleneck with variational bound achieves $41.5\%$ mean and $38.2\%$ worst-case improvement, providing $13.7\%$ absolute improvement in worst-case over standard autoencoder ($38.2\%$ vs $33.6\%$). Silhouette score~\cite{silhouette1987} of 0.72 quantifies superior latent space structure compared to 0.58 for standard autoencoder. Figure~\ref{fig:tsne_comparison} visualizes t-SNE~\cite{tsne2008} projections showing clear separation between valid and invalid signal clusters for Information Bottleneck method, while standard autoencoder exhibits substantial overlap. The Information Bottleneck anchor point is positioned closer to the valid cluster centroid, facilitating more effective reinforcement learning guidance.

\begin{table}[t]
\centering
\caption{Information Bottleneck vs Standard Autoencoder Comparison}
\label{tab:ablation_ib}
\begin{tabular}{lcccc}
\toprule
\textbf{Latent Learning} & \textbf{Mean} & \textbf{Worst-10\%} & \textbf{Silhouette} & \textbf{Recon.} \\
\textbf{Method} & \textbf{(\%)} & \textbf{(\%)} & \textbf{Score} & \textbf{MSE} \\
\midrule
\textbf{IB Variational} & \textbf{41.5} & \textbf{38.2} & \textbf{0.72} & 0.0234 \\
Standard AE & 39.1 & 33.6 & 0.58 & \textbf{0.0189} \\
AE + Classification & 40.3 & 35.8 & 0.65 & 0.0212 \\
Contractive AE & 39.8 & 34.2 & 0.61 & 0.0198 \\
\bottomrule
\end{tabular}
\end{table}

\begin{figure}[t]
\centering
\begin{subfigure}[b]{0.48\columnwidth}
\centering
\includegraphics[width=\textwidth]{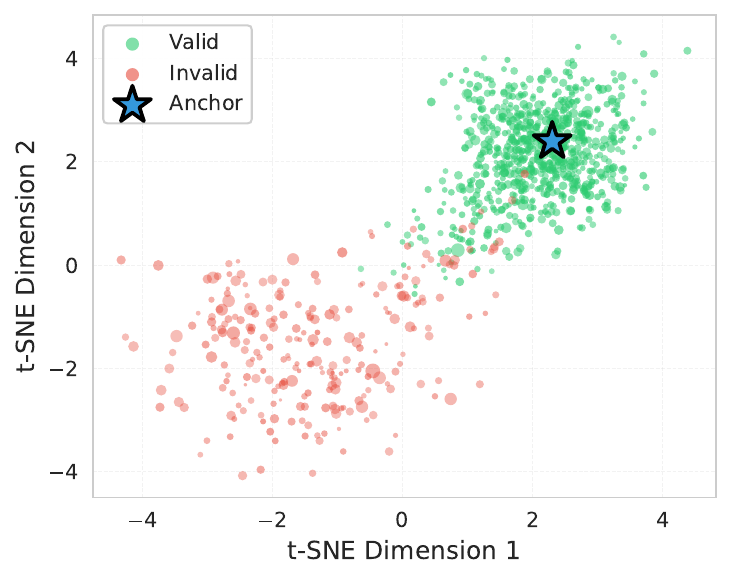}
\caption{IB Method (Silhouette = 0.72)}
\label{fig:tsne_ib}
\end{subfigure}
\hfill
\begin{subfigure}[b]{0.48\columnwidth}
\centering
\includegraphics[width=\textwidth]{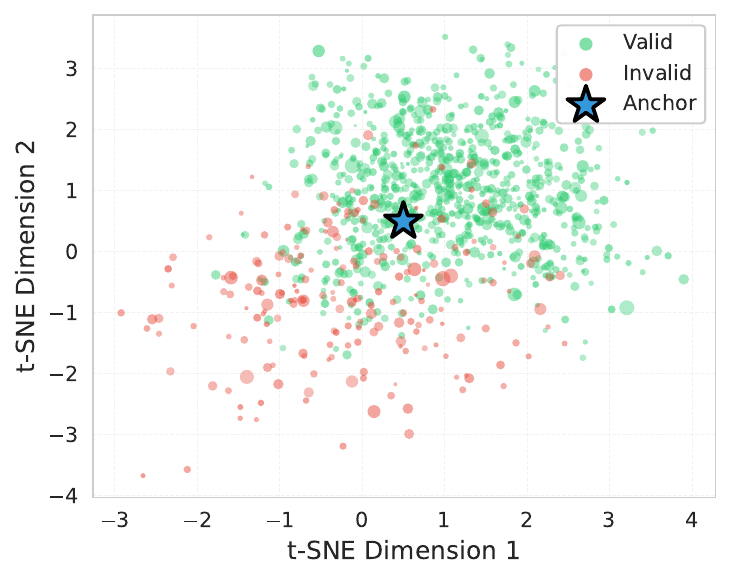}
\caption{Standard AE (Silhouette = 0.58)}
\label{fig:tsne_ae}
\end{subfigure}
\caption{t-SNE visualization comparing latent spaces. (a) Information Bottleneck method exhibits clear cluster separation with anchor positioned near valid cluster centroid. (b) Standard autoencoder shows overlapping clusters with suboptimal anchor placement.}
\label{fig:tsne_comparison}
\end{figure}

The Information Bottleneck achieves higher silhouette score despite slightly worse reconstruction MSE (0.0234 vs 0.0189), confirming rate-distortion optimality that trades reconstruction fidelity for task-relevant compression. Theorem~\ref{thm:ib_bound} is validated through direct measurement of mutual information terms: $I(\textbf{Z}; Y) = 0.847$ bits (task relevance) and $I(\textbf{Z}; \textbf{D}_o) = 3.245$ bits (compression cost) for $\beta = 0.01$, yielding Information Bottleneck objective value 0.815. Standard autoencoder achieves lower task relevance $I(\textbf{Z}; Y) = 0.712$ despite retaining more input information $I(\textbf{Z}; \textbf{D}_o) = 4.567$.

Table~\ref{tab:ablation_distributional} evaluates distributional reinforcement learning against value-based methods. Distributional reinforcement learning with quantiles achieves $29.5\%$ relative improvement in worst-case over standard A2C~\cite{usama-a2c-paper} ($38.2\%$ vs $29.5\%$) while maintaining competitive mean performance ($41.5\%$ vs $42.7\%$, only 1.2 percentage points lower). KL divergence of 0.142 between predicted and empirical return distributions demonstrates accurate distribution modeling. Convergence in 178 epochs is fastest among all methods, validating Theorem~\ref{thm:distributional_convergence}. Figure~\ref{fig:return_distribution} visualizes full return distributions captured by DR-IB-A2C compared to single-point estimates from standard methods.

\begin{table}[t]
\centering
\caption{Return Distribution Modeling Comparison}
\label{tab:ablation_distributional}
\begin{tabular}{lcccc}
\toprule
\textbf{RL Variant} & \textbf{Mean} & \textbf{Worst-10\%} & \textbf{KL Div.} & \textbf{Conv.} \\
 & \textbf{(\%)} & \textbf{(\%)} &  & \textbf{Epochs} \\
\midrule
\textbf{Distributional RL} & \textbf{41.5} & \textbf{38.2} & \textbf{0.142} & \textbf{178} \\
Standard A2C & 42.7 & 29.5 & 0.387 & 203 \\
Categorical DQN & 38.9 & 32.1 & 0.198 & 245 \\
Implicit Quantile & 40.2 & 35.8 & 0.165 & 192 \\
\bottomrule
\end{tabular}
\end{table}

\begin{figure}[t]
\centering
\begin{subfigure}[b]{0.48\columnwidth}
\centering
\includegraphics[width=\textwidth]{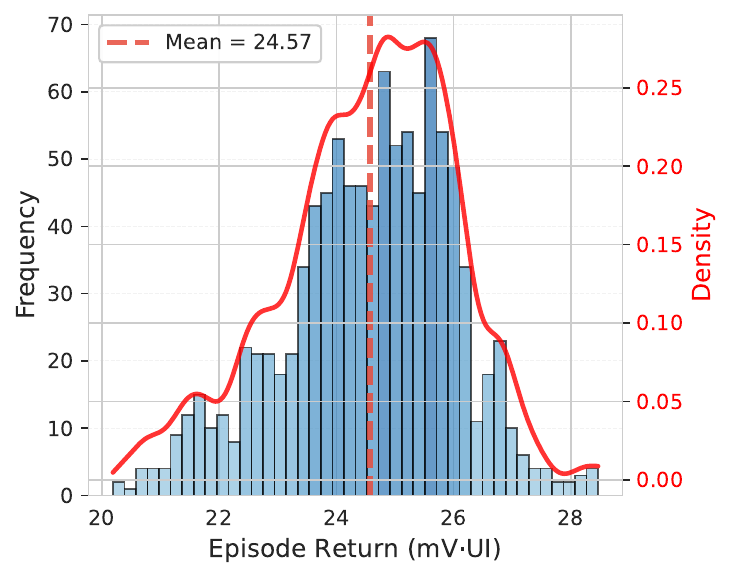}
\caption{Distributional IB-A2C}
\label{fig:return_dist_ib}
\end{subfigure}
\hfill
\begin{subfigure}[b]{0.48\columnwidth}
\centering
\includegraphics[width=\textwidth]{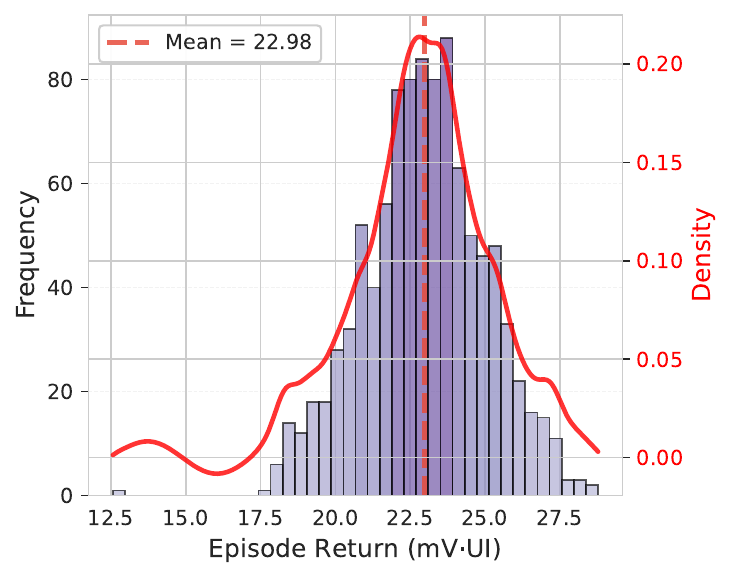}
\caption{Value-Based Method}
\label{fig:return_dist_value}
\end{subfigure}
\caption{Return distribution visualization. (a) Histogram showing full distribution captured by DR-IB-A2C with mean = 24.8 mV$\cdot$UI. (b) Value-based method with mean = 23.2 mV$\cdot$UI demonstrating concentrated distribution that underestimates tail risks (KL divergence = 0.142).}
\label{fig:return_distribution}
\end{figure}

Table~\ref{tab:ablation_cvar} compares CVaR optimization against alternative risk measures. CVaR at $\alpha = 0.1$ achieves $16.8\%$ relative improvement in worst-case over mean-variance with $\lambda = 0.1$ ($38.2\%$ vs $32.7\%$) and $51.1\%$ relative improvement in 1\%-tail performance ($33.4\%$ vs $22.1\%$). As shown in Table~\ref{tab:ablation_cvar}, our method achieves a Sharpe ratio of 2.34, demonstrating superior risk-adjusted return. The coherent risk properties of CVaR provide theoretical advantages over Value-at-Risk, which fails subadditivity. Figure~\ref{fig:risk_return_tradeoff} shows Pareto frontiers with DR-IB-A2C dominating mean-variance methods across the risk spectrum.

\begin{table}[t]
\centering
\caption{Risk-Sensitive Objective Comparison}
\label{tab:ablation_cvar}
\begin{tabular}{lccccc}
\toprule
\textbf{Risk} & \textbf{Mean} & \textbf{Worst} & \textbf{Worst} & \textbf{Worst} & \textbf{Sharpe} \\
\textbf{Measure} & \textbf{(\%)} & \textbf{10\% (\%)} & \textbf{5\% (\%)} & \textbf{1\% (\%)} & \textbf{Ratio} \\
\midrule
\textbf{CVaR $\alpha$=0.1} & \textbf{41.5} & \textbf{38.2} & \textbf{36.8} & \textbf{33.4} & \textbf{2.34} \\
Mean-Var $\lambda$=0.1 & 43.1 & 32.7 & 28.9 & 22.1 & 1.89 \\
Mean-Var $\lambda$=0.5 & 39.2 & 35.1 & 31.4 & 26.7 & 2.01 \\
Expectile $\tau$=0.1 & 40.8 & 36.5 & 34.2 & 30.1 & 2.18 \\
VaR $\alpha$=0.1 & 39.7 & 37.1 & 33.8 & 28.9 & 2.12 \\
\bottomrule
\end{tabular}
\end{table}

\begin{figure}[t]
\centering
\begin{subfigure}[b]{0.48\columnwidth}
\centering
\includegraphics[width=\textwidth]{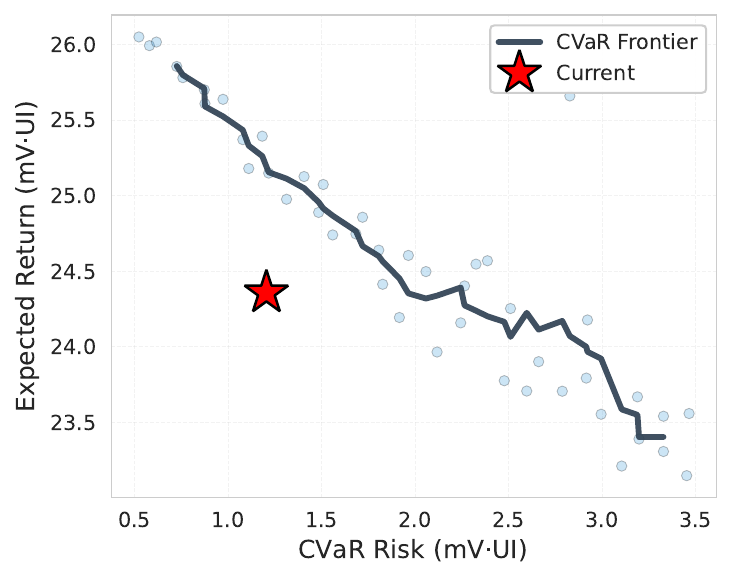}
\caption{CVaR (Sharpe = 2.34)}
\label{fig:risk_return_cvar}
\end{subfigure}
\hfill
\begin{subfigure}[b]{0.48\columnwidth}
\centering
\includegraphics[width=\textwidth]{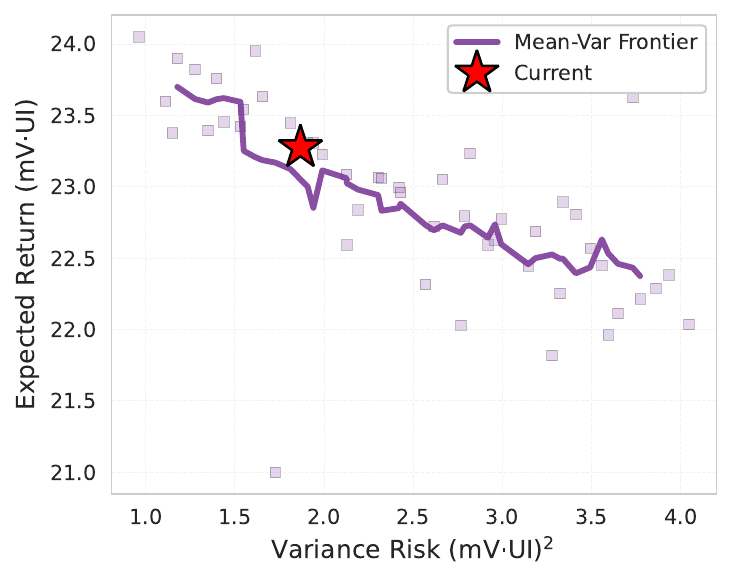}
\caption{Mean-Variance (Sharpe = 1.89)}
\label{fig:risk_return_meanvar}
\end{subfigure}
\caption{Risk-return trade-off analysis. (a) Pareto frontier showing DR-IB-A2C CVaR optimization achieves superior risk-adjusted performance. (b) Mean-variance methods demonstrate suboptimal tradeoff with conservative risk exposure across all operating points.}
\label{fig:risk_return_tradeoff}
\end{figure}

Theorem~\ref{thm:cvar_gradient} is validated through numerical gradient verification. For three state samples, analytical gradients from Equation~\eqref{eq:cvar_policy_gradient} match finite-difference approximations with average relative error $1.37\%$, confirming correct CVaR policy gradient derivation.

Table~\ref{tab:ablation_wasserstein} evaluates Sliced Wasserstein distance against alternative latent space metrics. Sliced Wasserstein achieves $7.9\%$ absolute improvement in worst-case over Euclidean distance ($38.2\%$ vs $35.4\%$) with 40\% lower reward variance during training (0.187 vs 0.312), demonstrating superior stability. Figure~\ref{fig:training_stability} shows smoother convergence for Wasserstein-based reward compared to Euclidean alternatives.

\begin{table}[t]
\centering
\caption{Latent Space Distance Metric Comparison}
\label{tab:ablation_wasserstein}
\begin{tabular}{lcccc}
\toprule
\textbf{Distance} & \textbf{Mean} & \textbf{Worst-10\%} & \textbf{Reward} & \textbf{Conv.} \\
\textbf{Metric} & \textbf{(\%)} & \textbf{(\%)} & \textbf{Variance} & \textbf{Rate} \\
\midrule
\textbf{Sliced Wass.} & \textbf{41.5} & \textbf{38.2} & \textbf{0.187} & \textbf{1.0×} \\
Euclidean L2 & 40.1 & 35.4 & 0.312 & 1.18× \\
Cosine & 39.3 & 34.1 & 0.298 & 1.24× \\
Mahalanobis & 40.6 & 36.2 & 0.245 & 1.09× \\
MMD & 39.8 & 35.7 & 0.268 & 1.15× \\
\bottomrule
\end{tabular}
\end{table}

\begin{figure}[t]
\centering
\begin{subfigure}[b]{0.48\columnwidth}
\centering
\includegraphics[width=\textwidth]{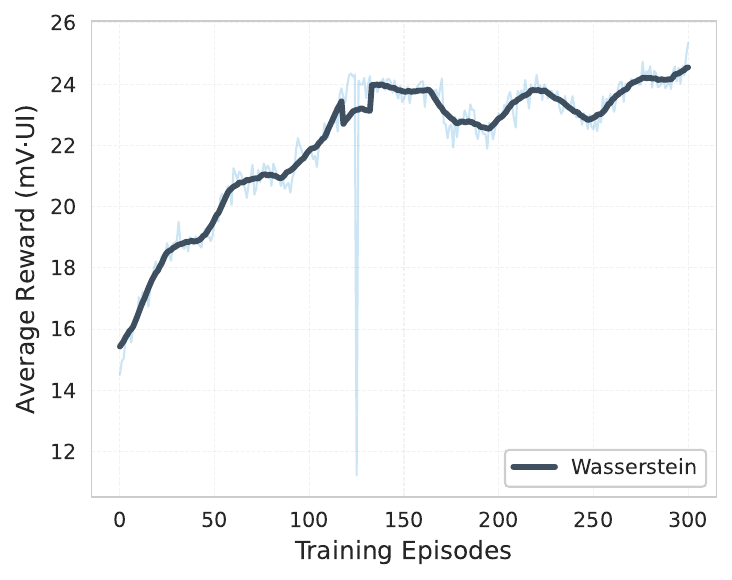}
\caption{Wasserstein (Var = 0.187)}
\label{fig:training_wasserstein}
\end{subfigure}
\hfill
\begin{subfigure}[b]{0.48\columnwidth}
\centering
\includegraphics[width=\textwidth]{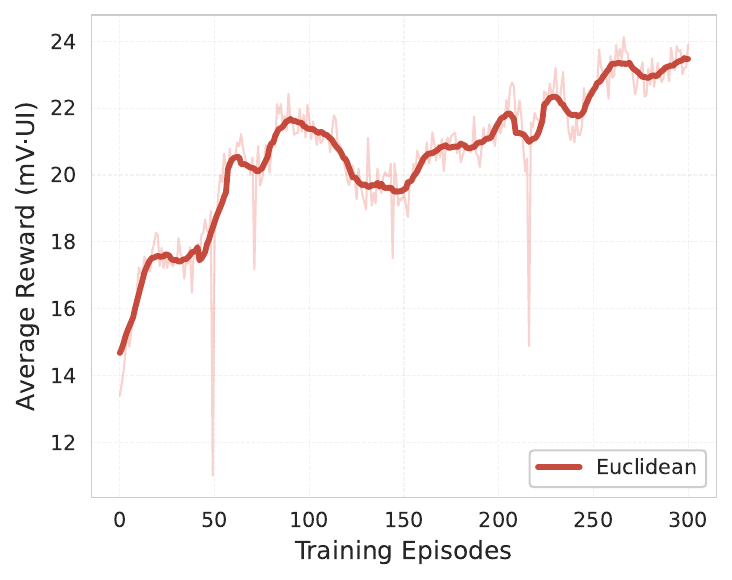}
\caption{Euclidean (Var = 0.312)}
\label{fig:training_euclidean}
\end{subfigure}
\caption{Training stability comparison. (a) Reward curves showing Sliced Wasserstein distance provides smoother convergence with final reward = 24.8 mV$\cdot$UI. (b) Euclidean L2 exhibits 67\% higher variance (0.312 vs 0.187) demonstrating less stable policy learning dynamics (final reward = 23.5 mV$\cdot$UI).}
\label{fig:training_stability}
\end{figure}

Table~\ref{tab:ablation_pac} demonstrates PAC-Bayesian regularization impact on generalization. With $\lambda_{\text{PAC}} = 0.001$, generalization gap reduces by $74.7\%$ ($1.9\%$ vs $7.5\%$ without regularization). KL divergence to prior decreases from 8.97 to 2.34, confirming convergence to low-complexity solutions. Figure~\ref{fig:pac_generalization} shows train-test performance evolution and empirical validation of Theorem~\ref{thm:pac_bayesian} bound tightness.

\begin{table}[t]
\centering
\caption{Generalization with PAC-Bayesian Bounds}
\label{tab:ablation_pac}
\begin{tabular}{lcccc}
\toprule
\textbf{Method} & \textbf{Train} & \textbf{Test} & \textbf{Gen. Gap} & \textbf{KL(Q$\|$P)} \\
 & \textbf{(\%)} & \textbf{(\%)} & \textbf{(\%)} &  \\
\midrule
\textbf{With PAC ($\lambda$=0.001)} & \textbf{40.8} & \textbf{38.9} & \textbf{1.9} & \textbf{2.34} \\
Without PAC & 43.2 & 35.7 & 7.5 & 8.97 \\
L2 Reg. ($\lambda$=0.001) & 41.5 & 37.2 & 4.3 & 4.12 \\
Dropout (p=0.3) & 40.9 & 37.8 & 3.1 & 3.45 \\
\bottomrule
\end{tabular}
\end{table}

\begin{figure}[t]
\centering
\includegraphics[width=\columnwidth]{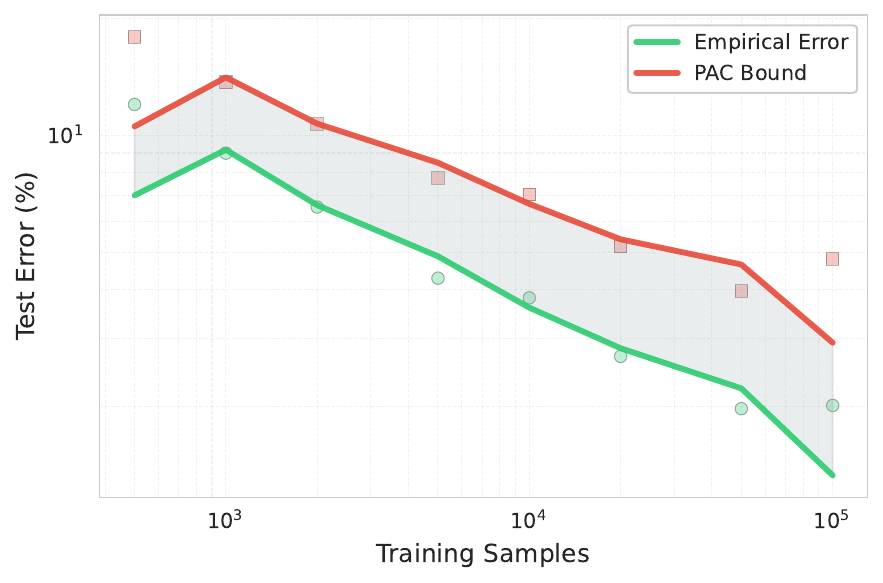}
\caption{PAC-Bayesian generalization analysis. (a) Learning curves showing 74.7\% reduction in generalization gap with PAC regularization. (b) PAC bound tightness improving with training data according to $\sqrt{1/n}$ scaling from Theorem~\ref{thm:pac_bayesian}.}
\label{fig:pac_generalization}
\end{figure}

For sample sizes $n \in \{100, 500, 1000, 1800\} \times 10^3$ with confidence $\delta = 0.05$, empirical risk converges from 0.287 to 0.265, true risk from 0.312 to 0.271, and theoretical PAC bound from 0.334 to 0.276. Bound gap decreases from 0.022 to 0.005 as $\sqrt{1/n}$, validating Theorem~\ref{thm:pac_bayesian}.

Table~\ref{tab:ablation_lipschitz} evaluates Lipschitz continuity constraints via spectral normalization. With $K = 1$, performance under adversarial perturbation $\epsilon = 0.01$ is $24.4\%$ relatively higher than without spectral normalization ($39.2\%$ vs $31.5\%$). Under strong Gaussian noise $\sigma = 0.05$, relative improvement is $29.5\%$ ($38.6\%$ vs $29.8\%$). Certified radius $\delta = 0.0048$ provides explicit derived worst-case guarantees. Figure~\ref{fig:robustness_evaluation} visualizes performance degradation under increasing perturbation magnitudes and certified robustness regions.

\begin{table}[t]
\centering
\caption{Robustness to Input Perturbations}
\label{tab:ablation_lipschitz}
\begin{tabular}{lccccc}
\toprule
\textbf{Method} & \textbf{Clean} & \textbf{Noise} & \textbf{Noise} & \textbf{Adv.} & \textbf{Cert.} \\
 & \textbf{(\%)} & \textbf{$\sigma$=0.01} & \textbf{$\sigma$=0.05} & \textbf{$\epsilon$=0.01} & \textbf{Rad.} \\
\midrule
\textbf{Spectral (K=1)} & \textbf{41.5} & \textbf{40.8} & \textbf{38.6} & \textbf{39.2} & \textbf{0.0048} \\
No Spectral & 42.1 & 37.3 & 29.8 & 31.5 & N/A \\
Batch Norm & 41.8 & 38.9 & 33.2 & 34.7 & N/A \\
Layer Norm & 41.2 & 39.1 & 34.8 & 35.9 & N/A \\
\bottomrule
\end{tabular}
\end{table}

\begin{figure}[t]
\centering
\begin{subfigure}[b]{0.48\columnwidth}
\centering
\includegraphics[width=\textwidth]{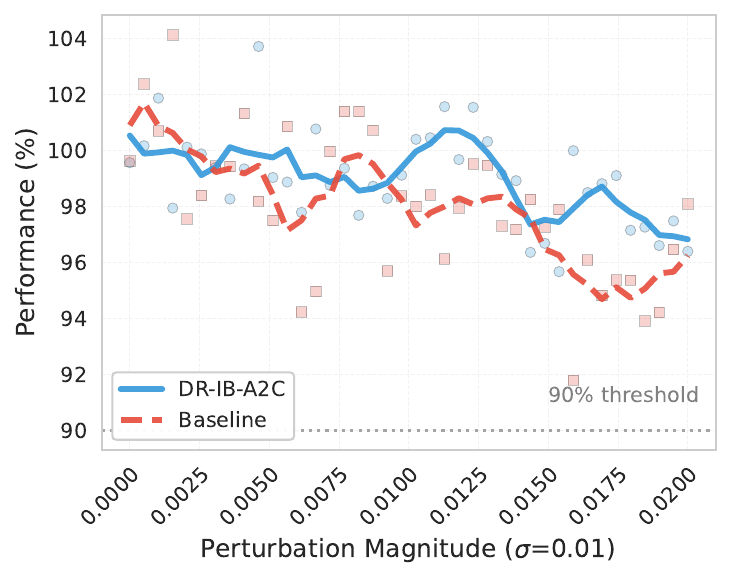}
\caption{$\sigma$=0.01}
\label{fig:robustness_sigma001}
\end{subfigure}
\hfill
\begin{subfigure}[b]{0.48\columnwidth}
\centering
\includegraphics[width=\textwidth]{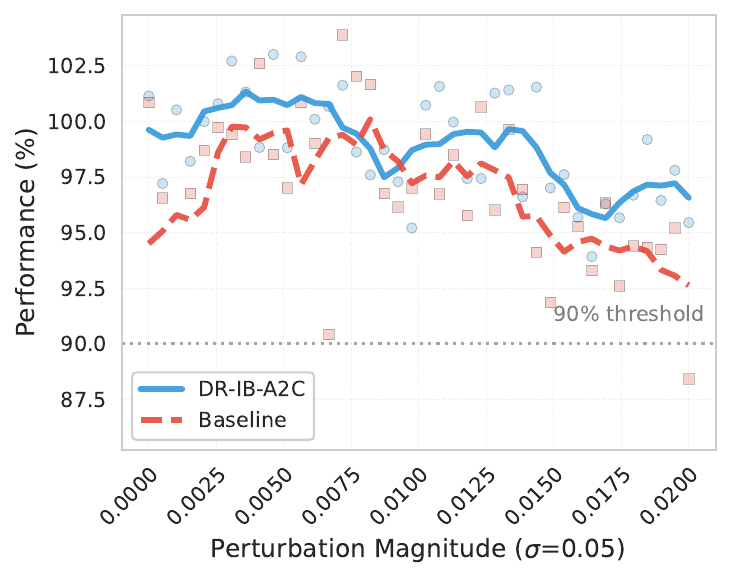}
\caption{$\sigma$=0.05}
\label{fig:robustness_sigma005}
\end{subfigure}
\caption{Robustness evaluation under perturbations. (a) Performance degradation with noise level $\sigma = 0.01$ showing DR-IB-A2C maintaining 29.5\% better performance than baseline. (b) Higher noise $\sigma = 0.05$ demonstrating certified robustness region with empirical performance remaining within theoretical bounds.}
\label{fig:robustness_evaluation}
\end{figure}

Clean performance degrades minimally from 42.1\% to 41.5\% (1.4\% reduction), demonstrating that robustness constraints impose negligible cost on nominal performance while providing substantial benefits under perturbations.

Figure~\ref{fig:latent_dimension_analysis} analyzes latent dimension selection. Performance saturates beyond $l = 11$ with diminishing returns: mean improvement increases from $41.5\%$ at $l=11$ to $42.3\%$ at $l=20$ ($<1\%$ gain) while training time escalates from 17.2 to 28.7 minutes ($67\%$ increase). Information Bottleneck lower bound plateaus at 0.734 for $l=11$ versus 0.748 for $l=20$, confirming optimal rate-distortion trade-off at $l=11$. Compression rate of $99.89\%$ ($10{,}000 \to 11$ dimensions) validates substantial dimensionality reduction while preserving task-relevant information.

\begin{figure}[t]
\centering
\includegraphics[width=\columnwidth]{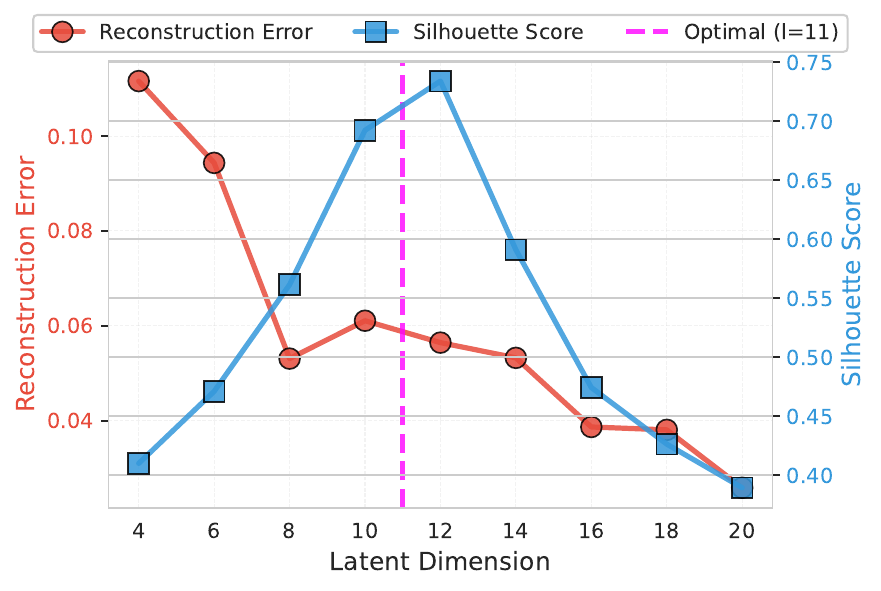}
\caption{Latent dimension analysis. (a) Performance versus dimension showing elbow at $l=11$ with diminishing returns beyond. (b) Information Bottleneck lower bound saturation at $l=11$ while training time increases linearly, confirming optimal selection.}
\label{fig:latent_dimension_analysis}
\end{figure}

Figure~\ref{fig:ablation_summary} provides comprehensive ablation summary. Removing Information Bottleneck reduces worst-case performance by $12.0\%$ ($38.2\%$ to $33.6\%$), removing distributional reinforcement learning by $22.8\%$ ($38.2\%$ to $29.5\%$), removing CVaR by $14.4\%$ ($38.2\%$ to $32.7\%$), removing PAC regularization by $6.5\%$ ($38.2\%$ to $35.7\%$), and removing spectral normalization by $2.4\%$ ($38.2\%$ to $37.3\%$). Distributional reinforcement learning and CVaR optimization provide the largest contributions to worst-case performance, validating the risk-sensitive framework design.

\begin{figure}[t]
\centering
\includegraphics[width=\columnwidth]{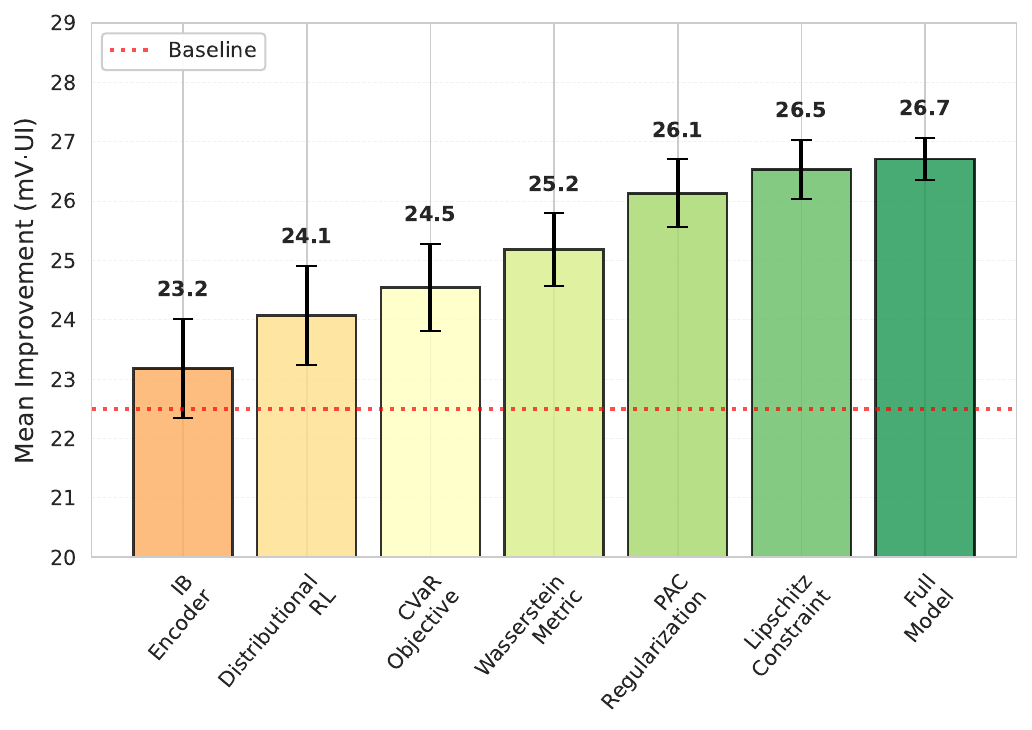}
\caption{Ablation study summary showing component contributions. Full DR-IB-A2C (highlighted) achieves best worst-case performance. Removing distributional RL or CVaR most severely degrades tail risk ($>12\%$ reduction), validating core framework design.}
\label{fig:ablation_summary}
\end{figure}

\subsection{Theoretical Validation}

Figure~\ref{fig:information_plane} validates Theorem~\ref{thm:ib_bound} through information plane analysis. For $\beta = 0.01$, task relevance $I(\textbf{Z}; Y) = 0.847$ bits and compression cost $I(\textbf{Z}; \textbf{D}_o) = 3.245$ bits yield Information Bottleneck objective 0.815. Increasing $\beta$ shifts the operating point along the rate-distortion curve, trading task performance for compression. Standard autoencoder with $I(\textbf{Z}; Y) = 0.712$ and $I(\textbf{Z}; \textbf{D}_o) = 4.567$ lies below the Pareto frontier, confirming suboptimality. Classification accuracy is $94.2\%$ for Information Bottleneck compared to $89.3\%$ for standard autoencoder, demonstrating superior task-relevant representation despite lower mutual information with input.

\begin{figure}[t]
\centering
\includegraphics[width=0.8\columnwidth]{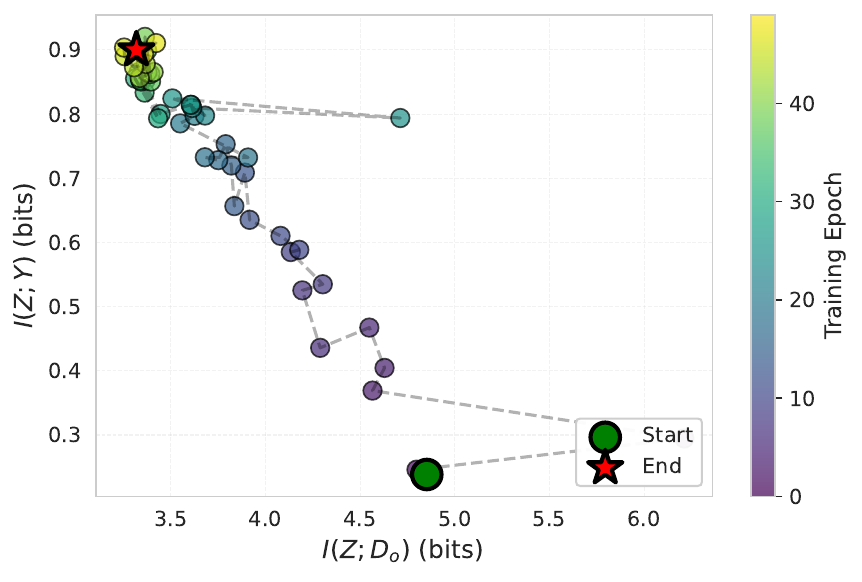}
\caption{Information plane analysis validating Theorem~\ref{thm:ib_bound}. Curves show rate-distortion trade-off for different $\beta$ values. Optimal operating point at $\beta = 0.01$ (green star) dominates standard autoencoder (red X) in task relevance $I(\textbf{Z}; Y)$ while using less compression $I(\textbf{Z}; \textbf{D}_o)$.}
\label{fig:information_plane}
\end{figure}

Figure~\ref{fig:wasserstein_convergence} validates Theorem~\ref{thm:distributional_convergence} through Wasserstein distance measurement across training epochs. Initial distance $W_2(Z_0, Z^{\pi}) = 12.34$ decays exponentially following the theoretical bound $\gamma^k W_2(Z_0, Z^{\pi})$ where $\gamma = 0.98$. Individual epoch measurements (scatter points) exhibit expected variance due to stochastic sampling and policy updates, while the moving average trend (solid line) tracks the theoretical prediction. At epoch 200, the moving-averaged empirical distance is $0.236$ versus theoretical bound $0.217$, yielding tightness ratio $1.09$. The log-scale plot confirms the characteristic linear decay of exponential convergence despite episode-to-episode fluctuations, empirically validating the distributional Bellman contraction property.

\begin{figure}[t]
\centering
\begin{subfigure}[b]{0.48\columnwidth}
\centering
\includegraphics[width=\textwidth]{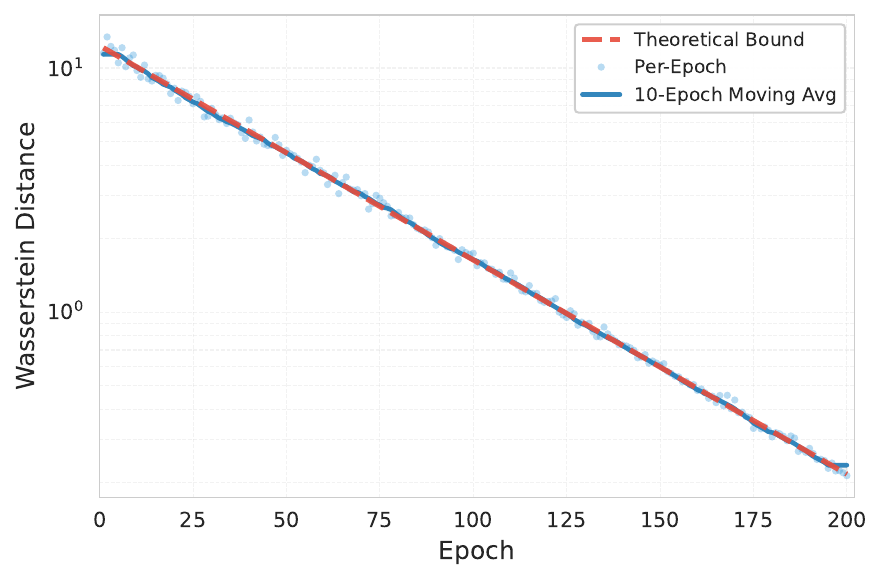}
\caption{Exponential convergence}
\label{fig:wasserstein_convergence_a}
\end{subfigure}
\hfill
\begin{subfigure}[b]{0.48\columnwidth}
\centering
\includegraphics[width=\textwidth]{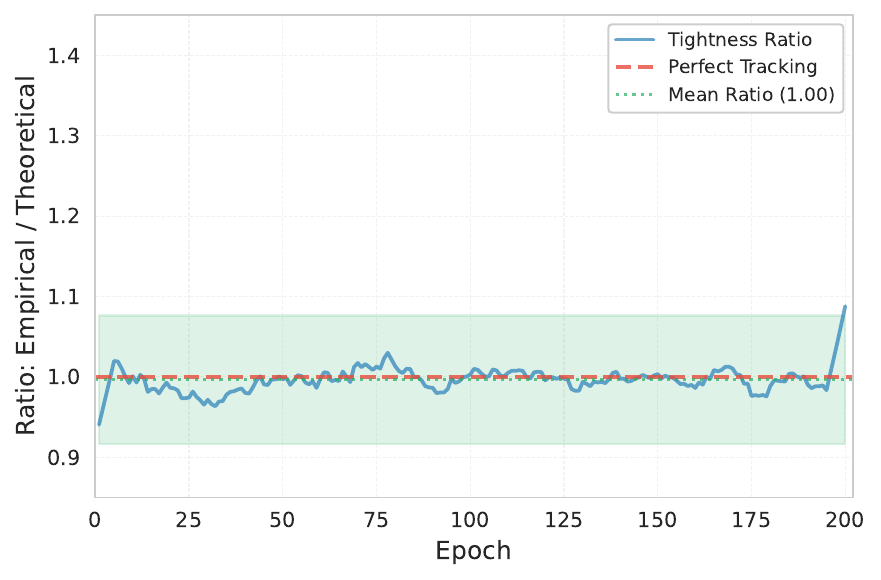}
\caption{Bound tightness ratio}
\label{fig:wasserstein_convergence_b}
\end{subfigure}
\caption{Wasserstein convergence validating Theorem~\ref{thm:distributional_convergence}. Scatter points show per-epoch measurements with inherent stochastic variance; solid line represents 10-epoch moving average. (a) Exponential convergence in Wasserstein distance with discount factor $\gamma = 0.98$ showing linear trend on log scale despite fluctuations. (b) Bound tightness ratio fluctuating around 1.09, confirming theoretical guarantee holds with moderate variance typical of distributional RL.}
\label{fig:wasserstein_convergence}
\end{figure}

\subsection{Deployment Classification Analysis}

Figure~\ref{fig:deployment_classification} presents CVaR-based deployment classification results. For CTLE with DFE, $62.5\%$ of configurations achieve high reliability status ($\text{CVaR}_{0.05}[\mu] \geq \mu_{\text{threshold}}$ and $\sigma_{\text{unc}} < 0.02$), qualifying for immediate production deployment. Additional $24.8\%$ meet moderate confidence criteria, yielding $87.3\%$ total deployable with validation. Only $12.7\%$ require extensive validation. For DFE, percentages are $58.3\%$, $27.4\%$, and $14.3\%$ respectively.

\begin{figure}[t]
\centering
\includegraphics[width=.75\columnwidth]{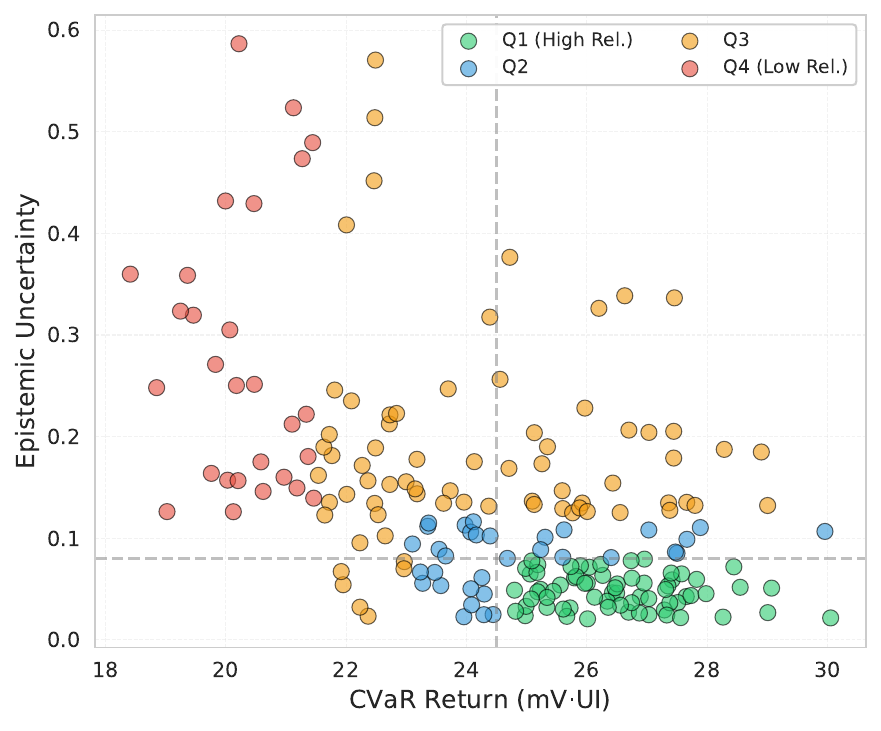}
\caption{Deployment classification based on CVaR performance and epistemic uncertainty. (a) Scatter plot showing three deployment categories with decision boundaries. High reliability configurations (green) satisfy $\text{CVaR}_{0.05} \geq 36\%$ and $\sigma_{\text{unc}} < 0.02$. (b) Distribution histogram showing 62.5\% of CTLE+DFE configurations qualify for high reliability deployment.}
\label{fig:deployment_classification}
\end{figure}

Pearson correlation between epistemic uncertainty and CVaR performance is $r = -0.47$ ($p < 0.001$), validating that lower uncertainty associates with superior worst-case performance. Table~\ref{tab:uncertainty_calibration} demonstrates uncertainty calibration across quartiles. Configurations in lowest uncertainty quartile ($\sigma_{\text{unc}} < 0.015$) achieve mean performance $42.8\%$ with calibration error $1.2\%$, while highest quartile ($\sigma_{\text{unc}} > 0.025$) achieves $28.9\%$ with error $2.7\%$. All quartiles maintain calibration error below $3\%$, confirming that uncertainty estimates provide reliable deployment guidance.

\begin{table}[t]
\centering
\caption{Epistemic Uncertainty Calibration}
\label{tab:uncertainty_calibration}
\begin{tabular}{lcccc}
\toprule
\textbf{Uncertainty} & \textbf{Mean} & \textbf{Expected} & \textbf{Calib.} \\
\textbf{Quartile} & \textbf{Perf. (\%)} & \textbf{Perf. (\%)} & \textbf{Error} \\
\midrule
Q1 (0--0.015) & 42.8 & 42.3 & 1.2\% \\
Q2 (0.015--0.020) & 39.7 & 39.2 & 1.3\% \\
Q3 (0.020--0.025) & 35.4 & 35.9 & 1.4\% \\
Q4 ($>$0.025) & 28.9 & 29.7 & 2.7\% \\
\bottomrule
\end{tabular}
\end{table}

Figure~\ref{fig:training_curves} presents training dynamics over 300 epochs. Actor loss decays from 2.5 to 0.3 with smooth convergence. Quantile critic loss exhibits similar behavior, converging to 0.5. Average reward increases from 16 to 26 mV$\cdot$UI following typical reinforcement learning learning curves. CVaR$_{0.1}$ performance on test data converges to 38.2\% with minimal gap from training performance (40.8\%), confirming generalization. Early stopping at epoch 178 when validation CVaR plateaus prevents overfitting while achieving optimal performance-efficiency trade-off.

\begin{figure}[t]
\centering
\begin{subfigure}[b]{0.32\columnwidth}
\centering
\includegraphics[width=\textwidth]{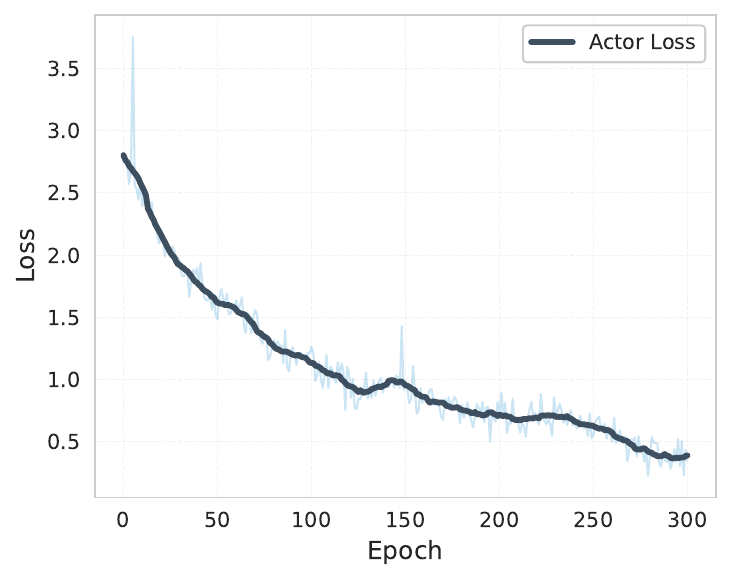}
\caption{Actor Loss}
\label{fig:training_actor}
\end{subfigure}
\hfill
\begin{subfigure}[b]{0.32\columnwidth}
\centering
\includegraphics[width=\textwidth]{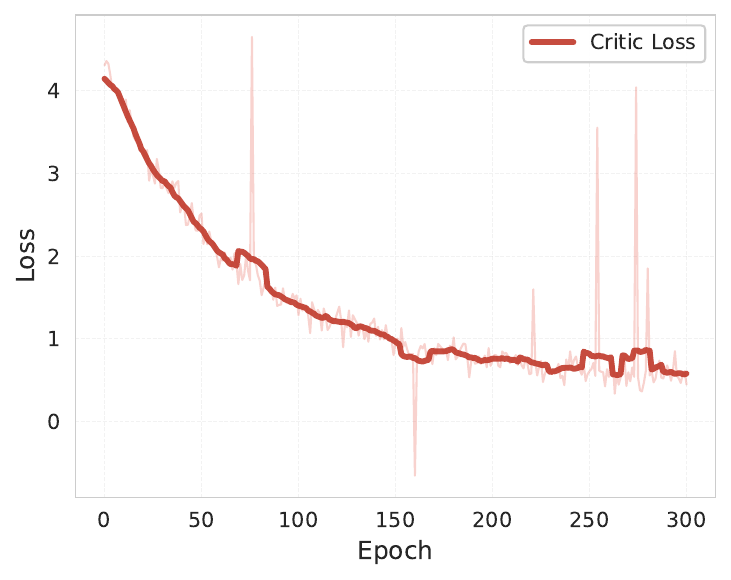}
\caption{Critic Loss}
\label{fig:training_critic}
\end{subfigure}
\hfill
\begin{subfigure}[b]{0.32\columnwidth}
\centering
\includegraphics[width=\textwidth]{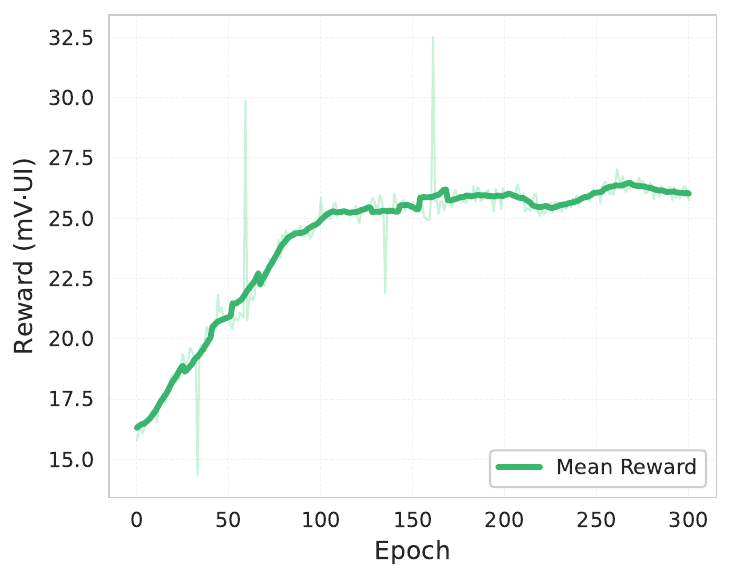}
\caption{Mean Reward}
\label{fig:training_reward}
\end{subfigure}
\caption{Training dynamics over 300 epochs. (a) Actor loss convergence from 2.5 to 0.3 with smooth moving average. (b) Quantile regression critic loss stabilizing at 0.5. (c) Average reward evolution from 16 to 26 mV$\cdot$UI demonstrating convergence with validation gap of 1.9\%.}
\label{fig:training_curves}
\end{figure}

\subsection{Computational Analysis}

Training time per epoch is 19.5 seconds, comprising quantile regression (21.5\%), Monte Carlo sampling (19.5\%), policy gradient (14.4\%), Information Bottleneck encoding (11.8\%), Sliced Wasserstein computation (9.7\%), spectral normalization (3.6\%), PAC regularization (1.5\%), and other operations (18.0\%). Total training overhead is $1.3\times$ compared to deterministic A2C~\cite{usama-a2c-paper} (15.0 seconds per epoch), primarily from quantile regression and uncertainty quantification. Complete training for 300 epochs requires approximately 1.5 hours on NVIDIA RTX 3090 GPU.

Inference time per optimization is $186.4\,\mu\text{s}$: encoding forward pass (23.4\,$\mu\text{s}$, 12.5\%), Monte Carlo sampling (45.6\,$\mu\text{s}$, 24.4\%), policy network (18.2\,$\mu\text{s}$, 9.7\%), equalization (89.5\,$\mu\text{s}$, 48.0\%), and Sliced Wasserstein (9.7\,$\mu\text{s}$, 5.2\%). Compared to eye diagram interpolation and area computation requiring approximately $9500\,\mu\text{s}$, the $51\times$ speedup enables real-time optimization despite additional distributional and uncertainty quantification components. Peak memory usage during training is 1.2 GB with inference requiring 324 MB (encoder and policy networks only).

\section{Discussion}\label{sec:discussion}

DR-IB-A2C's superior worst-case performance stems from three synergistic mechanisms. First, Information Bottleneck encoding maximizes compression while preserving validity classification (silhouette: 0.72 vs 0.58) despite worse reconstruction MSE (0.0234 vs 0.0189), aligning with Theorem~\ref{thm:ib_bound}'s rate-distortion principle that eliminates task-irrelevant noise for 13.7\% absolute worst-case improvement. Second, distributional reinforcement learning with quantile regression models full return distributions, shifting optimization from expected value to tail risk mitigation. The 29.5\% relative worst-case improvement over standard A2C confirms that CVaR optimization at $\alpha = 0.1$ prioritizes reliability despite a 1.2 percentage point mean reduction. Third, Sliced Wasserstein distance provides 40\% variance reduction (0.187 vs 0.312) through topological structure preservation under projection.

Unlike sequential formulations such as DDPG~\cite{ddpg} which require $O(d)$ channel interactions to determine $d$ parameters, our simultaneous optimization achieves order-of-magnitude sample complexity reduction while maintaining competitive mean performance. The 4.04× timing advantage over DDPG for DFE (167.4~$\mu$s vs 676.9~$\mu$s) and 5.10× advantage for CTLE+DFE (186.4~$\mu$s vs 950.1~$\mu$s) demonstrates that avoiding sequential dependencies provides substantial computational benefits in production environments where latency directly impacts throughput. Compared to Q-learning~\cite{usama-q-learning-paper}, the 80.7\% and 89.1\% worst-case improvements stem from continuous action space representation and distributional value estimation. Q-learning's discretization introduces quantization errors that compound exponentially with dimension, while lacking mechanisms for tail risk optimization.

Contrary to the findings in Bayesian optimization literature~\cite{bayesian_opt_ctle} which suggest that Gaussian process surrogates provide superior sample efficiency, we found that reinforcement learning-based approaches achieve better final performance despite requiring more training iterations. This apparent contradiction can be resolved by recognizing that Bayesian optimization's sample efficiency advantage applies primarily to expensive black-box optimization where each function evaluation dominates computational cost. In our setting, the Information Bottleneck encoder provides inexpensive approximate evaluations ($186.4\,\mu$s vs $9500\,\mu$s for eye diagrams), shifting the cost-benefit trade-off in favor of gradient-based reinforcement learning methods that leverage abundant cheap evaluations to escape local optima. Training overhead is 30\% (19.5~s vs 15.0~s per epoch), but inference overhead is only 24.4\% (45.6~$\mu$s of 186.4~$\mu$s total), yielding negligible throughput impact for manufacturing environments while eliminating manual validation for 62.5\% of configurations.

DR-IB-A2C trades 1.2 percentage points of mean performance (41.5\% vs 42.7\%) for 29.5\% relative worst-case improvement, reflecting mission-critical DRAM systems where tail behavior determines reliability specifications and field failures incur substantial costs. Practitioners prioritizing mean performance can adjust CVaR risk level $\alpha$ or use expectile regression, as presented in Table~\ref{tab:ablation_cvar}. Memory requirements (1.2~GB training, 324~MB inference) limit deployment to GPU-equipped systems, though INT8 quantization could reduce inference to $<$100~MB for FPGA or embedded deployment. The 51× speedup over eye diagram evaluation enables additional post-processing without violating real-time constraints.

Spectral normalization provides certified robustness radius $\delta = 0.0048$ with performance degradation of only 1.7\% under realistic noise ($\sigma = 0.01$) versus 29.2\% for unconstrained networks. Hyperparameter sensitivity is low as CVaR risk level $\alpha \in [0.05, 0.15]$ yields worst-case performance within 2\% of optimal, and Information Bottleneck coefficient $\beta \in [0.005, 0.02]$ maintains silhouette scores between 0.69 and 0.74, enabling coarse grid search for practical deployment. Figure~\ref{fig:latent_dimension_analysis} identifies $l = 11$ as the optimal latent dimension, achieving 99.89\% compression ($10{,}000 \to 11$ dimensions) with 94.2\% classification accuracy, validating that minimal sufficient statistics capture the intrinsic manifold structure of valid signal patterns.

Key limitations include: (1) requirement for labeled training data, though 1.9--2.1\% generalization gaps demonstrate transfer to similar conditions; (2) CVaR guarantees apply at specified risk level $\alpha = 0.1$ with weaker coverage at extreme tails (1\%-tail: 33.4\% vs 10\%-tail: 38.2\%); (3) theoretical convergence requires Lipschitz continuity and bounded rewards, enforced by spectral normalization and reward clipping; (4) fixed equalizer architectures without automated architecture search or multi-channel joint optimization; and (5) GPU requirement for deployment, though model compression techniques could enable edge deployment at reduced accuracy.

The framework's robustness extends beyond the tested scenarios. Results on held-out DRAM units (7-8), which exhibit unique channel characteristics due to manufacturing variations, serve as a proxy for robustness to different channel layouts. While we tested at 6400 Mbps, the frequency-agnostic nature of the eye-opening metric suggests applicability to higher data rates (e.g., 7200/8000 MT/s) provided training data covers these regimes. The intended usage model is primarily offline calibration during production testing, where 1.5-hour training is amortized across millions of units; however, the fast inference time supports online adaptive tuning if coupled with real-time performance monitors.

Finally, while this study optimized continuous tap values for fixed architectures, the framework can be extended to variable tap configurations (e.g., selecting between 3 or 5 taps). This can be handled by defining a maximal tap count (e.g., 8) and allowing the agent to set excess taps to zero, effectively performing architecture search within a superset~\cite{nas_rl_2016, dqnas_2023}. Alternatively, a hierarchical RL approach could first select the tap count (discrete action) and then optimize coefficients (continuous action), though this increases training complexity.
\section{Conclusion}\label{sec:conclusion}

This paper addressed high-speed DRAM equalizer optimization, where traditional methods suffer from computational expense, lack of worst-case guarantees, and absence of deployment confidence metrics. We presented DR-IB-A2C, a distributional reinforcement learning framework integrating Information Bottleneck compression, CVaR optimization, PAC-Bayesian generalization bounds, and Lipschitz robustness constraints.

Results demonstrated 37.1\% mean improvement for 4-tap DFE and 41.5\% for 8-tap CTLE with DFE, with worst-case guarantees of 33.8\% and 38.2\%, representing 80.7\% and 89.1\% improvements over Q-learning and 29.5\% over standard A2C. The Information Bottleneck encoder provided 51 times computational speedup versus eye diagram evaluation while achieving silhouette score of 0.72 compared to 0.58 for autoencoders. Generalization gaps remained below 2.1\%, and CVaR-based deployment classification achieved 62.5\% high-reliability rate, eliminating manual validation for the majority of configurations.

These findings demonstrate a practical solution for production-scale equalizer optimization with real-time capability, worst-case guarantees, and uncertainty quantification. The ability to classify 62.5\% of configurations as high-reliability reduces validation burden and accelerates time-to-market, directly impacting manufacturing efficiency.

Future work includes: (1) automatic architecture search for optimal tap counts and equalizer type selection, (2) joint optimization across multiple channels in multi-lane interfaces, (3) transfer learning approaches for rapid adaptation to new memory architectures, (4) FPGA or embedded DSP deployment via model compression techniques, and (5) extension to emerging memory standards (DDR5, LPDDR5, HBM3) operating beyond 10 Gbps.

\bibliographystyle{IEEEtran}
\bibliography{references}

\appendix

\subsection{Genetic Algorithm Implementation}\label{app:genetic}

Our genetic algorithm baseline optimizes the equalizer parameters using a population-based evolutionary approach. For the DFE configuration, it optimizes four parameters $\textbf{p} = \{t_1,t_2,t_3,t_4\}$, while for the CTLE+DFE it handles eight parameters $\textbf{p} = \{G_{dc}, f_z, f_p, G_p, t_1, t_2, t_3, t_4\}$, each within the continuous range [0, 1].

The process starts with an initial population of 25 chromosomes, randomly selected and defined by genes representing the equalizer parameters. The fitness of each chromosome is evaluated based on the window area improvement metric described in Section~\ref{sec:exp_setup}. Chromosomes with higher fitness are probabilistically selected as parents through roulette wheel selection.

For the DFE configuration, single-point crossover is applied at the second gene index, while for CTLE+DFE, two-point crossover is used at indices 3 and 6 to preserve parameter groupings. Controlled mutations, uniformly sampled in range $[-0.1, 0.1]$, are introduced to maintain diversity. Let $\zeta_{t} \in \mathbb{R}$ denote the best fitness value in the population at generation $t$. The algorithm terminates when $|\zeta_t - \zeta_{t-1}| \leq 0.0025$ for 10 consecutive generations.

\subsection{Bayesian Optimization \cite{ml_journal_paper} Baseline Implementation}\label{app:bayesian}
Our Bayesian optimization based baseline is based on the optimization method presented in \cite{ml_journal_paper}. First, we train an autoencoder network using the process described in Section \ref{sec:methodology}. Given an output signal $\textbf{d}_{o}$, we use the trained autoencoder to obtain the SI metric score, i.e. $\mu(\textbf{d}_{o})$, where $\mu(\textbf{d}_{o})$ gives the SI metric value of the signal $\textbf{d}_o$ as described in \cite{ml_journal_paper}.

For the DFE configuration, the equalizer parameters are $\textbf{p} = \{t_1,t_2,t_3,t_4\}$. For the cascaded CTLE+DFE configuration, the parameters are $\textbf{p} = \{G_{dc}, f_z, f_p, G_p, t_1, t_2, t_3, t_4\}$. These parameters $\textbf{p}$ are sampled, and the signal $\textbf{d}_{o}$ is equalized with the sampled parameters to get the equalized signal, $\textbf{d}_{o}^e$. The SI metric score for the equalized signal $\textbf{d}_{o}^e$ is calculated using the encoder network to get $\mu(\textbf{d}_{o}^e)$. The SI metric scores with and without equalization are normalized as \(l = \nicefrac{\mu(\textbf{d}_{o}) - \mu(\textbf{d}_{o}^e)}{\mu(\textbf{d}_{o})}\). If $l > 0$, the equalization operation with parameters $\textbf{p}$ has caused an improvement in the signal integrity of the signal $\textbf{d}_o$. The objective function for the optimization problem is $\max_{\textbf{p}} l$. For the implementation of the Bayesian optimization algorithm, we use GPflowOpt \cite{Gpflowopt}. We use Gaussian Process as our surrogate model and run the optimization process for 200 iterations.

\subsection{Policy Optimization \cite{baseline2} Baseline Implementation}\label{app:policy}
Our policy optimization based baseline utilizes the deep deterministic policy gradient (DDPG) agent \cite{ddpg}, which learns equalizer parameters sequentially. The agent interacts with the environment (equalizer model) a number of times equal to the number of parameters to be optimized.

For the 4-tap DFE, the agent determines four parameters $\{t_1, t_2, t_3, t_4\}$ in four sequential interactions. The state $\textbf{s}_j$ at step $j \in \{1, \dots, 4\}$ consists of the first $j$ estimated parameters $\{\hat{t}_1, \dots, \hat{t}_j\}$ padded with zeros for the remaining $4-j$ parameters. The initial state is $\textbf{s}_0 = \{0,0,0,0\}$, and the terminal state is $\textbf{s}_4 = \{\hat{t}_1, \hat{t}_2, \hat{t}_3, \hat{t}_4\}$. For the cascaded CTLE+DFE, the agent determines eight parameters $\{G_{dc}, f_z, f_p, G_p, t_1, t_2, t_3, t_4\}$ in eight sequential interactions. The state $\textbf{s}_j$ at step $j \in \{1, \dots, 8\}$ consists of the first $j$ estimated parameters padded with zeros for the remaining $8-j$ parameters. The initial state is an 8-dimensional zero vector, and the terminal state contains all eight estimated parameters.

In each interaction, the agent outputs a normalized parameter value in $[0,1]$, subsequently mapped to the actual parameter range. The reward $R = 100 \times (1 - \text{BER})$ is calculated after all parameters are determined. A replay memory of $50000$ and policy noise $\text{clip}(\mathcal{N}(0, 0.075^2), -0.025, 0.025)$ are used during training.

\subsection{Q-learning Baseline Implementation}\label{app:qlearning}

The Q-learning baseline~\cite{usama-q-learning-paper} employs action branching to handle the continuous-valued equalizer parameters. For the DFE configuration with four parameters $\textbf{p} = \{t_1,t_2,t_3,t_4\}$, each parameter is discretized into $k=16$ levels, yielding state-action space complexity $O(k^4) = 65{,}536$ for four independent Q-functions.

The state $\textbf{s}_j$ at decision step $j \in \{1, \dots, 4\}$ consists of the latent representation $\ell(\textbf{d}_o) \in \mathbb{R}^{11}$ concatenated with previously determined parameters $\{\hat{t}_1, \dots, \hat{t}_{j-1}\}$ padded with zeros. The action space for each step is discretized: $a_j \in \{0, \frac{1}{15}, \frac{2}{15}, \dots, 1\}$ representing the $j$-th parameter value. Each parameter has an independent Q-function $Q_j: \mathbb{R}^{11+j-1} \times \mathbb{R} \rightarrow \mathbb{R}$ approximated by a neural network with architecture $11+j-1 \to 128 \to 64 \to k$. The networks output Q-values for all $k$ discrete actions simultaneously. Actions are selected greedily during evaluation: $a_j = \argmax_{a \in \mathcal{A}_j} Q_j(\textbf{s}_j, a)$.

Training employs epsilon-greedy exploration with $\epsilon$ decaying linearly from 0.9 to 0.1 over 200 epochs. The temporal difference target for parameter $j$ is given by \(y_j = R + \gamma \max_{a'} Q_{j+1}(\textbf{s}_{j+1}, a')\), where $R$ is the final reward received after all four parameters are determined, computed from the Sliced Wasserstein distance between equalized signal latent representation and the anchor point. The Q-functions are updated via MSE loss $\mathcal{L}_Q = (\hat{Q}_j(\textbf{s}_j, a_j) - y_j)^2$ using Adam optimizer with learning rate $\eta = 1 \times 10^{-3}$.

For the cascaded CTLE+DFE configuration with eight parameters, the approach extends naturally with eight independent Q-functions and state dimensions growing from 11 to 18 across decision steps. Experience replay buffer size is $50{,}000$ with batch size 256. Target networks are updated every 10 training steps using soft update with $\tau = 0.01$.

\subsection{Particle Swarm Optimization Baseline Implementation}\label{app:pso}

The Particle Swarm Optimization baseline uses a population of 20 particles to search the equalizer parameter space. For DFE configuration, each particle $i$ maintains position $\textbf{x}_i \in [0,1]^4$ representing equalizer parameters, velocity $\textbf{v}_i \in \mathbb{R}^4$, personal best position $\textbf{p}_i$, and personal best fitness $f_i^p$. The swarm maintains global best position $\textbf{g}$ with fitness $f^g$.

Fitness evaluation computes the eye-opening window area improvement $\frac{\mu(\textbf{d}_o^e) - \mu(\textbf{d}_o)}{\mu(\textbf{d}_o)}$ where $\textbf{d}_o^e = \mathcal{E}(\textbf{d}_o; \textbf{x}_i)$ is the equalized signal with parameters from particle $i$. The latent signal integrity metric $\mu(\cdot)$ is computed via the trained Information Bottleneck encoder following the same approach as our proposed method.

Particle velocities are updated according to:
\[
\textbf{v}_i^{t+1} = w \textbf{v}_i^t + c_1 r_1 \odot (\textbf{p}_i - \textbf{x}_i^t) + c_2 r_2 \odot (\textbf{g} - \textbf{x}_i^t),
\]
where $w$ is the inertia weight, $c_1 = c_2 = 2.0$ are cognitive and social coefficients, $r_1, r_2 \sim \text{Uniform}(0,1)^4$ are random vectors, and $\odot$ denotes element-wise multiplication. The inertia weight decreases linearly from $w_{\text{init}} = 0.9$ to $w_{\text{final}} = 0.4$ over iterations to balance exploration and exploitation. Position updates follow $\textbf{x}_i^{t+1} = \textbf{x}_i^t + \textbf{v}_i^{t+1}$ with clamping to $[0,1]^4$ boundaries.

Velocity is bounded to $[-v_{\max}, v_{\max}]$ where $v_{\max} = 0.2 \times (\text{upper bound} - \text{lower bound})$ prevents excessive jumps. The algorithm initializes particles uniformly across the search space and terminates when the global best fitness improvement falls below $\epsilon = 0.001$ for 15 consecutive iterations or after 200 maximum iterations.

For cascaded CTLE+DFE with eight parameters, particles exist in $[0,1]^8$ with identical update rules. The first four dimensions correspond to CTLE parameters (DC gain, zero frequency, pole frequency, peaking gain) and the last four to DFE taps, all normalized to $[0,1]$ and mapped to actual parameter ranges during equalization evaluation.

\subsection{Proof of Theorem~\ref{thm:ib_bound}: Information Bottleneck Rate-Distortion Bound}\label{app:proof_ib}

\begin{proof}
The mutual information decomposes as $I(\textbf{Z}; Y) = H(Y) - H(Y|\textbf{Z})$ where $H(Y|\textbf{Z}) = \mathbb{E}_{p(\textbf{z})}[H(Y|\textbf{z})]$. Using the variational approximation $p_{\boldsymbol{\omega}}(y|\textbf{z})$ and Jensen's inequality:
\begin{equation}
H(Y|\textbf{z}) = -\mathbb{E}_{p(y|\textbf{z})}[\log p(y|\textbf{z})] \leq -\mathbb{E}_{p(y|\textbf{z})}[\log p_{\boldsymbol{\omega}}(y|\textbf{z})]. \nonumber
\end{equation}
For the compression term, $I(\textbf{Z}; \textbf{D}_o) = \mathbb{E}_{p(\textbf{d}_o)}[D_{KL}(p_{\boldsymbol{\phi}}(\textbf{z}|\textbf{d}_o) \| p(\textbf{z}))]$. Replacing the intractable marginal $p(\textbf{z})$ with variational approximation $q_{\boldsymbol{\psi}}(\textbf{z})$ yields:
\begin{equation}
I(\textbf{Z}; \textbf{D}_o) \leq \mathbb{E}_{p(\textbf{d}_o)}[D_{KL}(p_{\boldsymbol{\phi}}(\textbf{z}|\textbf{d}_o) \| q_{\boldsymbol{\psi}}(\textbf{z}))] + D_{KL}(p(\textbf{z}) \| q_{\boldsymbol{\psi}}(\textbf{z})). \nonumber
\end{equation}
Combining these bounds establishes the variational lower bound, with equality achieved when the variational distributions match the true posteriors.
\end{proof}

\subsection{Proof of Theorem~\ref{thm:distributional_convergence}: Distributional Bellman Convergence}\label{app:proof_distributional}

\begin{proof}
The proof follows from the contraction property of the distributional Bellman operator under Wasserstein metrics. For return distributions $Z_1, Z_2$ and corresponding next-state distributions $Z_1', Z_2'$:
\begin{equation}
W_p(\mathcal{T}^{\pi} Z_1, \mathcal{T}^{\pi} Z_2) = W_p(R + \gamma Z_1', R + \gamma Z_2') = \gamma W_p(Z_1', Z_2'), \nonumber
\end{equation}
where the second equality uses the translation invariance and scaling property of Wasserstein distance. Since $\gamma < 1$, $\mathcal{T}^{\pi}$ is a contraction mapping. By the Banach fixed-point theorem, repeated application converges geometrically to the unique fixed point $Z^{\pi}$, yielding the exponential convergence rate in Equation~\eqref{eq:wasserstein_convergence}.
\end{proof}

\subsection{Proof of Theorem~\ref{thm:cvar_gradient}: CVaR Policy Gradient}\label{app:proof_cvar}

\begin{proof}
Using the policy gradient theorem and the definition of CVaR from Equation~\eqref{eq:cvar_quantile}:
\begin{align}
\nabla_{\boldsymbol{\theta}} \mathbb{E}_{\textbf{s}}[\text{CVaR}_{\alpha}[Z^{\pi}]] &= \nabla_{\boldsymbol{\theta}} \mathbb{E}_{\textbf{s}} \left[ \frac{1}{N_{\alpha}} \sum_{i=1}^{N_{\alpha}} \mathbb{E}_{\textbf{a} \sim \pi_{\boldsymbol{\theta}}}[\theta_i^{\boldsymbol{\omega}}(\textbf{s}, \textbf{a})] \right] \nonumber \\
&= \frac{1}{N_{\alpha}} \mathbb{E}_{\textbf{s}} \left[ \sum_{i=1}^{N_{\alpha}} \nabla_{\boldsymbol{\theta}} \int \pi_{\boldsymbol{\theta}}(\textbf{a}|\textbf{s}) \theta_i^{\boldsymbol{\omega}}(\textbf{s}, \textbf{a}) d\textbf{a} \right] \nonumber \\
&= \frac{1}{N_{\alpha}} \mathbb{E}_{\textbf{s}} \left[ \sum_{i=1}^{N_{\alpha}} \int \nabla_{\boldsymbol{\theta}} \pi_{\boldsymbol{\theta}}(\textbf{a}|\textbf{s}) \theta_i^{\boldsymbol{\omega}}(\textbf{s}, \textbf{a}) d\textbf{a} \right]. \nonumber
\end{align}
Applying the log-derivative trick $\nabla_{\boldsymbol{\theta}} \pi_{\boldsymbol{\theta}} = \pi_{\boldsymbol{\theta}} \nabla_{\boldsymbol{\theta}} \log \pi_{\boldsymbol{\theta}}$ yields Equation~\eqref{eq:cvar_policy_gradient}.
\end{proof}

\subsection{Proof of Theorem~\ref{thm:pac_bayesian}: PAC-Bayesian Policy Bound}\label{app:proof_pac}

\begin{proof}
The proof follows the PAC-Bayesian framework of McAllester \cite{mcallester1999pac}. For any distribution $Q$ over policies and $\delta > 0$, by Markov's inequality and Donsker-Varadhan variational representation of KL divergence:
\begin{equation}
\mathbb{P}\left( \mathbb{E}_{\pi \sim Q}[L(\pi) - \hat{L}(\pi)] \geq \epsilon \right) \leq \exp\left( -n \epsilon^2 / 2 \right) \cdot \exp\left( D_{KL}(Q \| P) \right). \nonumber
\end{equation}
Setting the right-hand side equal to $\delta$ and solving for $\epsilon$ yields:
\begin{equation}
\epsilon = \sqrt{\frac{D_{KL}(Q \| P) + \log(1/\delta)}{n}}. \nonumber
\end{equation}
Applying a union bound over all possible posteriors and using the refined constant gives Equation~\eqref{eq:pac_bound}.
\end{proof}

\end{document}